\documentclass{article}

\usepackage[a4paper]{geometry}

\usepackage{capt-of,palatino}

\usepackage[normalem]{ulem}

\usepackage{amsthm}

\usepackage{hyperref}

\usepackage{algorithm}

\usepackage{algpseudocode}

\usepackage{algorithmicx}

\usepackage{subcaption}
\DeclareCaptionSubType*{algorithm}

\DeclareCaptionLabelFormat{alglabel}{Alg.~#2}

\usepackage{amsmath, amssymb}

\usepackage{graphicx}

\usepackage{epstopdf}

\newtheorem{theorem}{Theorem}
\newtheorem{definition}{Definition}

\newtheorem{con}{Convention}

\numberwithin{equation}{section}

\author{\small{\textbf{Andreas Hula\footnote{Wellcome Trust Centre for Neuroimaging, London, United Kingdom }, P.Read Montague\footnote{Wellcome Trust Centre for Neuroimaging , London, United Kingdom - Human Neuroimaging Laboratory, Virginia Tech Carilion Research Institute, Roanoke, VA, United States -  Department of Physics, Virginia Polytechnic Institute and State University, Blacksburg, VA, United States  }, Peter Dayan\footnote{Gatsby Computational Unit, University College London, London, United Kingdom}}}\\ \footnote{Address all correspondence to Andreas Hula, \href{mailto:a.hula@ucl.ac.uk}{a.hula@ucl.ac.uk}}}

\title{
\textbf{Monte Carlo Planning method estimates planning horizons during interactive social exchange}}

\date{\today}

\begin{document}
\maketitle
\begin{abstract}
  Reciprocating interactions represent a central feature of all human
  exchanges. They have been the target of various recent experiments,
  with healthy participants and psychiatric populations engaging as dyads in
  multi-round exchanges such as a repeated trust task. Behaviour in such
  exchanges involves complexities related to each agent's preference
  for equity with their partner, beliefs about the partner's appetite
  for equity, beliefs about the partner's model of their partner, and so
  on. Agents may also plan different numbers of steps into the future.
  Providing a computationally precise account of the behaviour is an
  essential step towards understanding what underlies choices. A
  natural framework for this is that of an interactive partially
  observable Markov decision process  (IPOMDP). 
  However, the various
  complexities make IPOMDPs inordinately computationally challenging.
  Here, we show how to approximate the solution for the multi-round
  trust task using a variant of the Monte-Carlo tree search algorithm.
  We demonstrate that the algorithm is efficient and effective, and therefore can be
  used to invert observations of behavioural choices. We use generated
  behaviour to elucidate the richness and sophistication of interactive
  inference.
\end{abstract}

\begin{section}{Author Summary}

Agents interacting in games with multiple rounds must model
their partners' thought processes over extended time horizons. This poses a 
substantial computational challenge that has  restricted previous behavioural analyses.
By taking advantage of
recent advances in algorithms for planning in the face of uncertainty,
we demonstrate how these formal methods can be extended. We use a
well studied social exchange game called the trust task to illustrate the
power of our method, showing how agents with particular cognitive
and social characteristics can be expected to interact, and how to
infer the properties of individuals from observing their behaviour.
\end{section}

\begin{section}{Introduction}

Successful social interactions require individuals to understand the 
consequences of their actions on the future actions and beliefs of those around
them. To map these processes is a complex challenge in at least three different ways. The first is that
other peoples' preferences or
  utilities are not known exactly. Even if
the various components of the  utility functions are held in common, 
the actual values of the
  parameters of partners, e.g., their degrees of envy or guilt
\cite{Fehr, FehrGaechter, FehrAlt, FehrFisch, Camerer03, McCabeAlt}, could well differ. 
 This ignorance 
decreases through experience, and
  can be modeled using the framework of a partially observable Markov
  decision process  (POMDP). However, normal mechanisms for
  learning in POMDPs involve probing or running experiments, which has
  the potential cost of partners fooling each other.
The second complexity is represented by characterizing the form of the model agents have of others.
  In principle, agent A's model of agent B should include agent B's
  model of agent A; and in turn, agent B's model of agent A's
  model of agent B, and so forth.  The beautiful
theory of Nash equilibria \cite{Nash}, extended to the case of
incomplete information via so-called Bayes-Nash equilibria \cite{Harsanyi}
dispenses with this so-called cognitive hierarchy \cite{Costa, Camerer, Yoshida, Sanfey},
looking instead for an equilibrium solution. However, a wealth of work (see for instance 
\cite{Centipede}) has shown that people
deviate from Nash behaviour. It has been proposed
people instead model others to a
strictly limited, yet non-negligible, degree \cite{Costa, Camerer}.

The final complexity arises when we consider that although it is common
in experimental economics to create one-shot interactions, many of the
most interesting and richest aspects of  behaviour arise with multiple
rounds of interactions. Here, for concreteness, we consider the multi round
trust task, which is a social exchange game that has been used with
hundreds of pairs (dyads) of subjects, including both normal and
clinical populations \cite{Brooks2008,Misha2010,Chiu, Kishida2010, FehrCamerer}. This game has
been used to show that characteristics that only arise in multi-round
interactions such as defection (agent
  A increases their cooperation between two rounds; agent B responds by
  decreasing theirs) have
observable neural consequences that can be measured using functional
magnetic resonance imaging (fMRI) \cite{Brooks2005, Brooks2008, Ting2012, Lee, McCabe}.

The interactive POMDP (IPOMDP) \cite{IPOMDP} is a theoretical framework
that formalizes many of these complexities. It characterizes the
uncertainties about the utility functions and planning over multiple
rounds in terms of a POMDP, and constructs an explicit cognitive
hierarchy of models about the other (hence the moniker
'interactive'). This framework has previously been used with data from
the multi-round trust task \cite{Debbs2008, Ting2012}. However, solving
IPOMDPs is computationally extremely challenging, restricting those previous investigations 
to a rather minuscule degree of forward planning
(just two- out of what is actually a ten-round interaction).
Our main contribution is the adaptation of an efficient Monte Carlo tree search method, 
called partially observable Monte Carlo planning (POMCP) to IPOMDP problems. 
 Our second contribution is to
illustrate  this algorithm through examination of the multiround trust task. We
show characteristic patterns of  behaviour to be expected for subjects
with particular degrees of inequality aversion, other-modeling and planning
capacities, and consider how to invert observed  behaviour to make
inferences about the nature of subjects' reasoning capacities.
\end{section}

\begin{section}{Materials and Methods}

We first briefly review Markov decision processes (MDPs), their partially observable extensions (POMDPs), and the POMCP algorithm 
  invented to solve them approximately, but efficiently. These concern
  single agents. We then discuss IPOMDPs and the application of POMCP to 
  solving them when there are multiple agents. Finally, we describe the multi-round trust task.
\subsection{Partially Observable Markov Decision Processes}
\begin{center}
\includegraphics[width =4.4 in, height =3in]{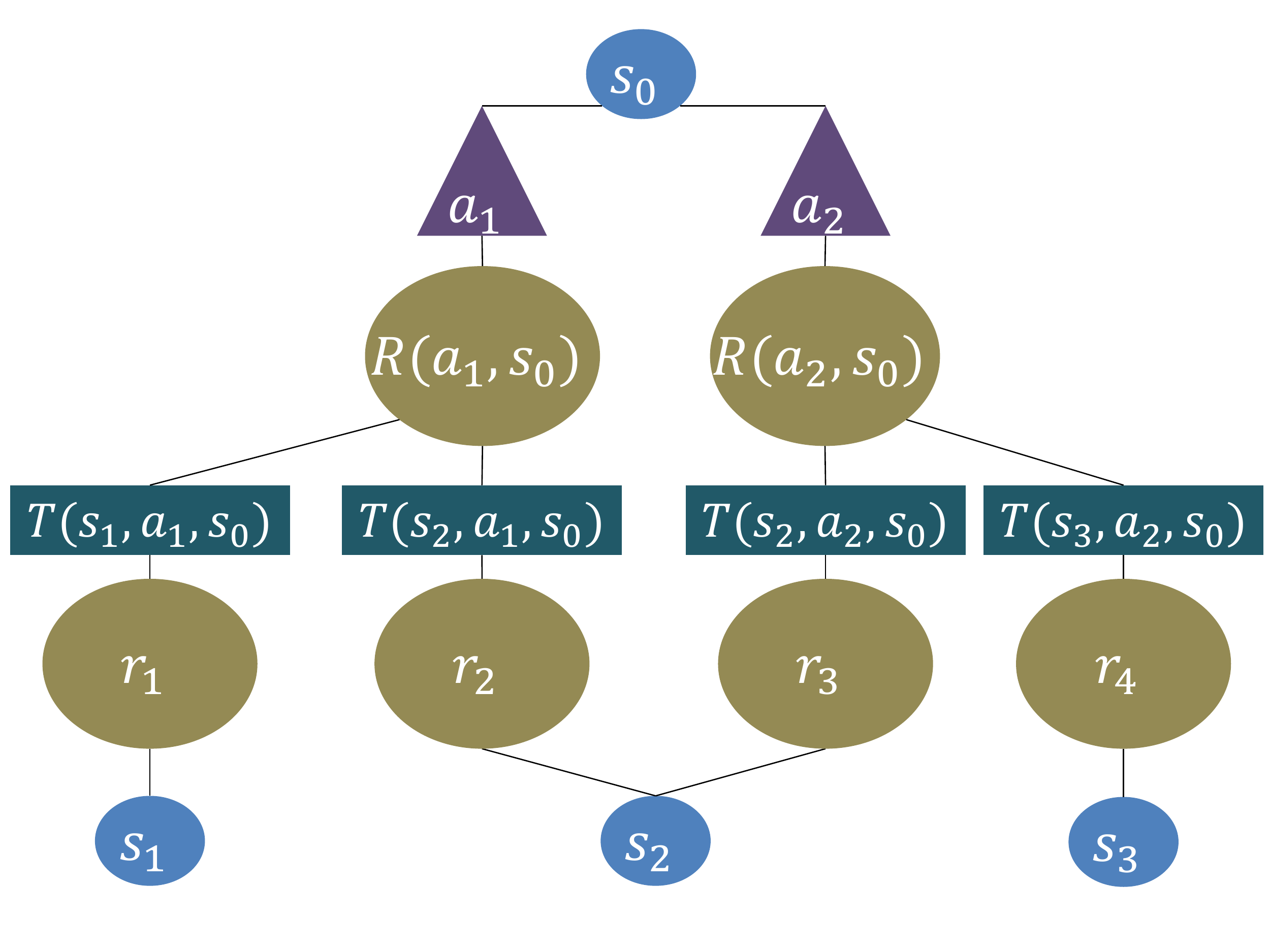}
\captionof{figure}{\small{A Markov decision process. The agent starts at state $s_0$ 
and has two possible actions $a_1$ and $a_2$. Exercising either, it can transition into three possible states, one of which 
($s_2$) can be reached through either action. Each state and action combination holds a particular reward expectation $R(a,s)$. Based on this information, 
the agent can choose an action and transitions with probability $\mathcal T (s,a,s_0)$ to a new state $s$, obtaining an actual reward $r$ in the process. The procedure is then 
repeated from the new state, with its' given action possibilities or else the decision process might end, depending on the given process.} }
\label{fig:MDPim}
\end{center}
A  Markov decision process (MDP) \cite{Puterman}
is defined by sets $\mathcal S$ of "states" and $\mathcal A$ of
"actions", and several components that evaluate and link the two, including transition
probabilities $\mathcal T$, and information $\mathcal R$ about possible
rewards.  States describe the position of the agent in the environment, and
determine which actions can be taken, accounting for, at least probabilistically,
the consequences for rewards and future states. Transitions between states 
are described by means of a collection of
transition probabilities $\mathcal T$, assigning to each possible state
$s\in\mathcal S$ and each possible action $a\in\mathcal A$ from that
state, a transition
probability distribution or measure $\mathcal T_{s\hat{s}}^a = \mathcal
T (\hat{s}, a,s):=\mathbb P[ \hat{s} | s, a]$ which 
encodes the likelihood of ending in state $\hat{s}$ after
taking action $a$ from state $s$.  
 The Markov property requires that the
transition (and reward probabilities) only depend on the current state
(and action), and are independent from the past events.
An illustration of these concepts can be found in figure \ref{fig:MDPim}.
\begin{center}
\includegraphics[width =4.4in, height = 3in]{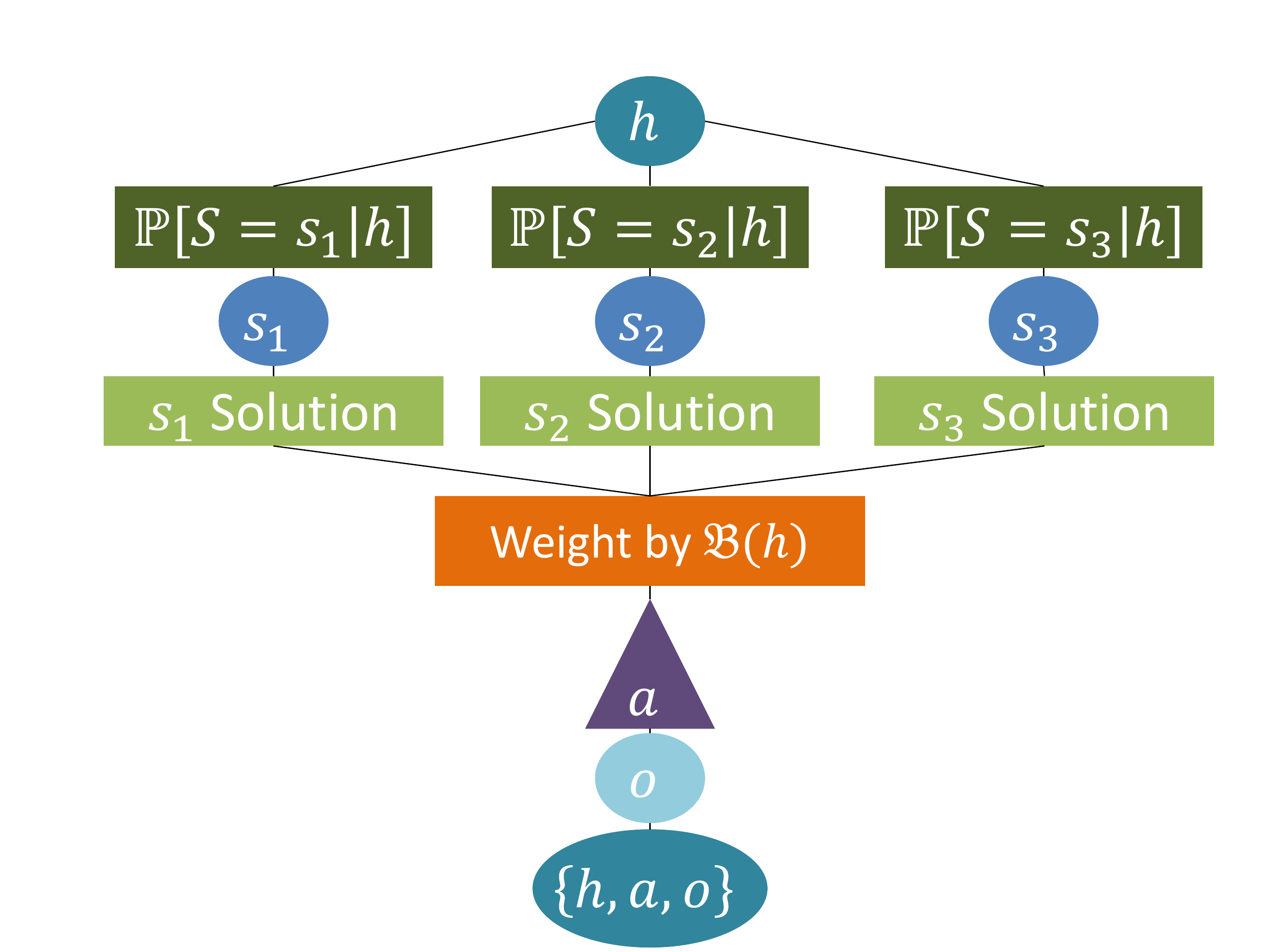}
\captionof{figure}{\small{A partially observable Markov decision process. Starting from a 
observed interaction history $h$, the agents use their belief state $\mathcal B (h)$, to determine how 
likely they find themselves in one of three possible actual states $s_1, s_2, s_3$. Any solution to the POMDP requires calculating all future action values depending on all possible future 
histories and the respective Belief states. 
Following this, an observation $o$ is obtained by the agent and the new history $\{h, a, o\}$ becomes the starting point for the next decision.
  }}
\label{fig:POMDPim}
\end{center}
 By contrast, in a partially observable MDP (i.e., a POMDP \cite{Kaelbling}), the agent
can also be uncertain about its state $s$. Instead, there is a set of
observations $o\in\mathcal O$ that incompletely pin down
states, depending on the observation probabilities 
$
\mathcal W_{\hat{s}o}^a = \mathcal W(o,a,\hat{s}) :=\mathbb P [o| \hat{s}, a].
$
 These report the
probability of observing $o$ when action $a$ has occasioned a transition
to state $\hat{s}$. See figure $\ref{fig:POMDPim}$ for an illustration of the concept.

We use the notation $s_t=s$, $a_t=a$ or $o_t=o$ to refer explicitly to
the outcome state, action or observation at a given time. 
 The {\it
  history\/} $h \in \mathcal H$ is the sequence of actions and
observations, wherein each action from the point of view of the agent
moves the time index ahead by $1$, $h_t := \{ o_0, a_0, o_1, a_1, \ldots,
a_{t-1}, o_t \}$.  Here $o_0$ may be trivial (deterministic or empty).
The agent can perform Bayesian inference to turn its history at time $t$
into a distribution $\mathbb P [ S_t=s_t | h_t]$ over its state at time
$t$, where $S_t$ denotes the random variable encoding the uncertainty about the 
current state at time $t$. 
This distribution is called its belief state $\mathcal B(h_t)$, with
$\mathbb P_{\mathcal B(h_t)}[S_t = s_t] := \mathbb P [ S_t = s_t | h_t].$ Inference depends
on knowing $\mathcal{T, W}$ and the distribution over the initial state
$S_0$, which we write as $\mathcal B(h_0)$.
Information about rewards $\mathcal R$ comprises a collection of utility
functions $r \in \mathcal R,r: \mathcal A\times \mathcal S \times \mathcal O \rightarrow
\mathbb R$, a discount function $\Gamma \in \mathcal R, \Gamma:
\mathbb N \rightarrow [0,1]$ \footnote{A more general definiton would be
 $\Gamma \in \mathcal R, \Gamma:\mathcal H \times \mathcal H \rightarrow [0,1]$, allowing it to be conditional on the precise present and 
future histories.} and a survival function 
$H \in \mathcal R, H: \mathbb N \times \mathbb N \rightarrow [0,1]$. The utility functions
determine the immediate gain associated with executing action $a$ at
state $s$ and observing $o$ (sometimes writing $r_t$ for the reward following the
$t^{\text{th}}$ action). From the utilities, we define the reward function 
$R: \mathcal A \times \mathcal S\rightarrow \mathbb R$, as the 
expected gain for 
taking action $a$ at state $s$ as $R(a,s) = \mathbb E[r(a,s,o)]$, where this 
expectation is taken over all possible observations $o$.
 Since we usually operate on histories, rather than fixed states, we define 
the expected reward from a given history $h$ as
$ R (a, h) := \sum_{s\in\mathcal S}  R( a, s)\mathbb P[ s | 
h].$ 
The discount function
weights the present impact of a future return, 
depending only on the separation between present and future. We use exponential discounting with 
a fixed number  $\gamma \in [0,1]$ to define our discount function:
\begin{equation}
\Gamma ({\tau}- t) =\gamma^{\tau - t}. \quad \forall\tau,  t \in \mathbb N, \tau \geq t.
\end{equation}

Additionally, we define $H$ such that 
$H(\tau, t)$ is $0$ for $\tau > K$ and $1$ otherwise. $K$ in general is a random stopping time. We call 
the second component $t$ the reference time of the survival function.

The survival function allows us to encode the planning horizon of an agent during decision making: 
If $H(\tau, t)$ is $0$ for $\tau-t > P$, we say that the local planning horizon at
$t$ is less than or equal to $P$.

The policy $\pi\in\Pi,
\pi (a,h): = \mathbb P[a| h]$
is defined as a mapping of histories to probabilities over possible
actions. Here $\Pi$ is called the set of admissible policies. For
convenience, we sometimes write the distribution function as
$\pi(h)$.
The value function of a fixed policy $\pi$
starting from present history $h_t$ is
\begin{equation}
V^{\pi}(h_t) :=  \sum_{\tau=t}^{\infty}  \gamma^{\tau - t} H(\tau, t) \mathbb E
[r_{\tau}|   \pi ,  h_{\tau} ]
\end{equation}
i.e., a sum of the discounted future expected rewards (note
  that $h_{\tau}$ is 
 a random variable here, not a fixed value). 
Equally, the state-action value is
\begin{equation}
Q^{\pi}(a, h_t) :=  R (a,h_t)+\sum_{\tau=t+1}^{\infty}  \gamma^{\tau-t}H(\tau,t)
\mathbb E [r_{\tau}|   \pi , h_{\tau}].
\end{equation}
\begin{definition}[Formal Definition - POMDP]
Using the notation of this section, a POMDP is defined as a tuple
$(\mathcal S, \mathcal A, \mathcal O, \mathcal T, \mathcal W, \mathcal
R, \Pi, \mathcal B_0)$  of components as outlined above.
\end{definition}

\begin{con}[Softmax Decision Making]\label{quant}
 A wealth of experimental work (for instance \cite{Norm, 2step, nogo}) has found 
that the choices of humans (and other animals) can be well described by softmax 
policies based on the agents' state-action values, to
 encompass the stochasticity of observed behaviour in real subject data. 
See \cite{quantal}, for a behavioural economics perspective 
and \cite{Yoshida} for a neuroscience perspective.
In view of using our model primarily for experimental analysis, we will base our discussion on 
the decision making rule: 
\begin{equation}\label{eq:softmax}
\pi(a,h)=\mathbb P  [ a | h] =  \frac {e^{ \beta Q^{\pi} (a , h)}}
{\sum_{b\in\mathcal A} e^{\beta Q^{\pi}(b , h)} } 
\end{equation}
where $\beta>0$ is called the inverse temperature parameter
  and controls how diffuse are the probabilities. The policy
\begin{equation}\label{eq:argmax}
\pi(a, h) = \left\{
\begin{array}{ll}
1 & \textrm{ if }  Q^{\pi} (a , h) = \max{  \{   Q^{\pi} (b , h) | b \in
  \mathcal A \}} \quad \textrm{(assuming this is unique)} \\
0 & \textrm{otherwise}
\end{array} \right.
\end{equation}
can be obtained as a limiting case for $\beta \rightarrow \infty$.
\end{con}
\begin{con}
From now on, we shall denote by $Q(a,h)$, the state-action value  $Q^{\pi}(a,h)$ with respect to the softmax policy.
\end{con}
\subsection{POMCP}\label{POMCP}

POMCP was introduced by  \cite{POMCP2010} as an efficient approximation
scheme for solving POMDPs. Here, for completeness, we describe the
algorithm; later, we adapt it to the case of an IPOMDP.

POMCP is a generative model-based sampling method for calculating history-action
values. That is, it builds a limited portion of the tree of future
histories starting from the current $h_t$, using a sample-based search
algorithm (called upper confidence bounds for trees (UCT); \cite{Szespari}) which provides guarantees as to
how far from optimal the resulting action can be, given a certain
number of samples (based on results in \cite{Auer1} and \cite{Auer2}). Algorithm ~\ref{POMCPFig} provides pseudo code for
the adapted POMCP algorithm. The procedure is presented 
schematically in figure $\ref{fig:MCTSim}$.

\begin{center}
\includegraphics[width =4.4 in, height = 2.2in]{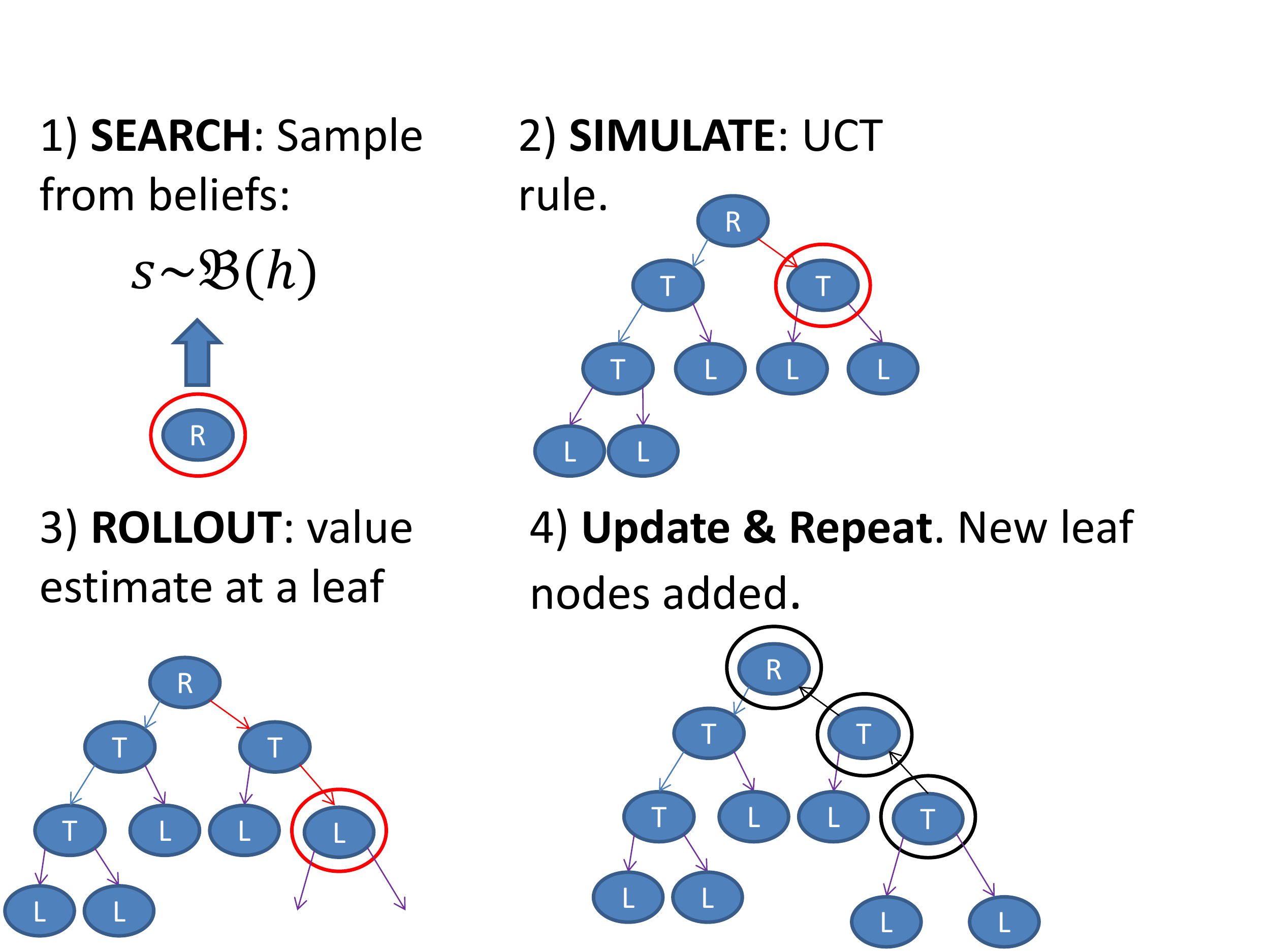}
\captionof{figure}{\small{Image of POMCP. The algorithm samples a state $s$ from the Belief state $\mathcal B (h)$ at the root R 
( $R$ representing the current history $h$), keeps this state $s$ fixed till step $4)$,
 follows UCT in already visited domains (labelled tree nodes $T$) and performs a rollout and back update when hitting a leaf (labelled $L$). 
Then step $1) - 4)$ is repeated until the specified number of simulations has been reached.}} 
\label{fig:MCTSim}
\end{center}

\begin{algorithm}
\caption{Partially Observable Monte Carlo Planning}\label{POMCPFig}
\begin{subalgorithm}{.5\textwidth}
\begin{algorithmic}

\Procedure{SEARCH}{h , t , n}
	\For {SIMULATIONS $= 1,\ldots, n$}  
		\State $k \gets t$
		\If{ $h_t= o_0$}
			\State $s \sim \mathcal B_0$ 
		\Else
			\State $s \sim \mathcal B(h_t)$
		\EndIf \\
		\hspace{0.96 cm}	SIMULATE $(s,h, t, k)$
	\EndFor\\
\hspace{0.5 cm}\Return $a \sim$ {\sc SoftUCT}
\EndProcedure
\Procedure{ROLLOUT}{$s$,$h$,$t$,$k$}
	\If{$ H (k, t) \leq0$}\\
		\hspace{1 cm}\Return 0
	\EndIf
		\State $a \sim \pi_{\textrm{rollout}}(h,\cdot)$
		\State $(s',o, r) \sim \mathcal G (s,a)$
		\State $h \gets \{h,a,o\}$ 
		\State $k \gets k +1$\\
 		\hspace{0.3 cm}\Return $r$+$\gamma $ROLLOUT$(s',h ,t, k)$
\EndProcedure
\end{algorithmic}
\end{subalgorithm}%
\begin{subalgorithm}{.5\textwidth}
\begin{algorithmic}
\Procedure{SIMULATE }{s,h ,t , k}
	\If{$ H (k , t) \leq 0$}\\
\hspace{1 cm}		\Return 0
	\EndIf
	\If{$h \not\in$ T}
		\ForAll {$a\in\mathcal A$}
			\State $T(ha)\gets (N_{\textrm{init}}(h,a), \tilde{Q}_{\textrm{init}}(a, h), \emptyset)$
		\EndFor\\
		\hspace{1 cm} \Return ROLLOUT $(s,h,t, k)$
	\EndIf \\
	\State $a \sim$ {\sc SoftUCT} 
	\State  $(s',o, r) \sim \mathcal G (s,a)$
	\State $h \gets \{h,a,o\}$
	\State $k \gets k+1$
 	\State $R\gets r$+$ \gamma $SIMULATE$(s',h,t,k)$
	\State $N(h) \gets N(h)+1$
	\State $N(h, a)\gets N(h,a)+1$
	\State $\tilde{Q}(a, h) \gets \tilde{Q}(a, h) + \frac{R -\tilde{Q}(a, h)}{N(h, a)}$\\
	\hspace{0.5 cm}\Return R
\EndProcedure
\end{algorithmic}
\end{subalgorithm}
\end{algorithm}

The algorithm is based on a tree structure $T$, wherein nodes $T(h) = (N(h), \tilde{Q}(h), \mathcal B (h))$ represent 
possible future histories explored by the algorithm, and are characterized by the number $N(h)$ of times history
 $h$ was visited in the simulation, the estimated value $\tilde{Q}(h)$ for visiting $h$ and 
the approximate belief state $\mathcal B (h)$ at $h$.
Each new node in $T$ is initialized with 
initial action exploration counts $N_{\textrm{init}}(h,a)=0$ for all 
possible actions $a$ from $h$ and 
an initial action value estimate $\tilde{Q}_{\textrm{init}}(h,a)=0$ for 
all possible actions $a$ from $h$
 and 
an empty belief state $\mathcal B (h) = \emptyset$.

The value $N(h)$ is then calculated from all actions counts 
from the node $N(h) = \sum_{a \in\mathcal A} N(h,a)$. 
$\tilde{Q}(h)$ denotes the mean of obtained values, for simulations 
starting from node $h$. $\mathcal B (h)$ can either be calculated analytically, 
if it is computationally feasible to apply Bayes theorem, or be approximated by the 
so called {\it root sampling} procedure (see below).

In terms of the algorithm, the generative model $\mathcal G (s,a)$ of the POMDP determines
$(s', o, r)\sim \mathcal G (s , a)$, the simulated reward,
observation and subsequent state for taking $a$ at $s$; $s$ itself is sampled from
the current history $h$.  Then, every (future) history
of actions and observations $h$ defines a node $T(h)$ 
in the tree structure $T$, which is characterized
by the available actions and their average
simulated action values $\tilde{Q}(a,h)$ under the policy {\sc SoftUCT}
at future states. 

If the node has been visited for the $N(h)^{\text{th}}$ time; with
action $a$ being taken for the $N(h,a)^{\text{th}}$ time, then the
average simulated value is updated (starting from $0$) using
sampled simulated rewards $R$  up to terminal time $K$,
when the current simulation/tree traversal ends as:
\begin{equation}
\tilde{Q}^{\textrm{new}}(a,h) =
\tilde{Q}^{\textrm{old}}(a, h) + \frac{1}{N(h ,a)}\left(
 R - \tilde{Q}^{\textrm{old}} (a, h)\right).
\end{equation}

The search algorithm has two decision rules, depending on whether a
traversed node has already been visited or is a leaf of the search tree.
In the former case, a decision is reached using {\sc SoftUCT} by
defining
\begin{equation}
\textrm{\sc SoftUCT} \qquad Q(a,h): = \tilde{Q}(a,h) + c\sqrt{\frac{\log
N(h)}{N(h,a)}} \qquad 
\mathbb P [a| h] = \frac{ e^{\beta(Q(a, h))}}{\sum_{b} e^{\beta (Q(b, h))} }.
\label{softUCT}
\end{equation}
where $c$ is a parameter that favors exploration (analogous to an
equivalent parameter in UCT).

If the node is new, a so-called "rollout" policy is used to provide a
crude estimate of the value of the leaf. This policy can be either very simple 
(uniform or $\epsilon -$greedy based on a very simple model) or specifically adjusted to the search space, 
in order to optimize performance. 

The rollout value estimate together with the {\sc SoftUCT} exploration rule 
is the core mechanism for efficient tree exploration. In this work, we only use an 
$\epsilon-$greedy mechanism, as is described in the section on the multi round 
trust game. 

Another innovation in POMCP that underlies its dramatically
superior performance is called {\it root sampling.} 
This procedure allows to form the belief state at later states, as long 
as the initial belief state $\mathcal B_0$ is known. This means that,
although it is necessary to perform inference to draw samples from the
belief state at the root of the search tree, one can then use each sample as
if it was (temporarily) true, without performing inference at states that are deeper
in the search tree to work out the new transition probabilities that
pertain to the new belief states associated with the histories at those
points. The reason for this is that the probabilities of getting to the
nodes in the search tree represent exactly what is necessary to compensate
for the apparent inferential infelicity \cite{POMCP2010}-- i.e., the search tree performs as a
probabilistic filter. The technical details of the root sampling procedure can be found in 
\cite{POMCP2010}.

In the presence of analytically tractable updating rules (or at least analytically tractable approximations) the belief state at a new node can instead be calculated 
by Bayes' theorem. In the case for the multi round trust game below, we follow the approximating updating rule in ~\cite{Ting2012}. 

\subsection{Interactive Partially Observable Markov Decision Processes}

An Interactive Partially Observable Markov Decision Process (IPOMDP) is
a multi agent setting in which the actions of each agent may observably
affect the distribution of expected rewards for the other agents.

Since IPOMDPs may be less familiar than POMDPs, we provide more detail
about them; consult \cite{IPOMDP} for a complete reference
formulation  and \cite{Wunder} for an excellent discussion and 
extension.

We define the IPOMDP such that the decision making process of each agent becomes a standard (albeit large) POMDP, allowing the direct application of POMDP methods to IPOMDP problems.

\begin{center}
\includegraphics[width =4.4 in, height = 3in]{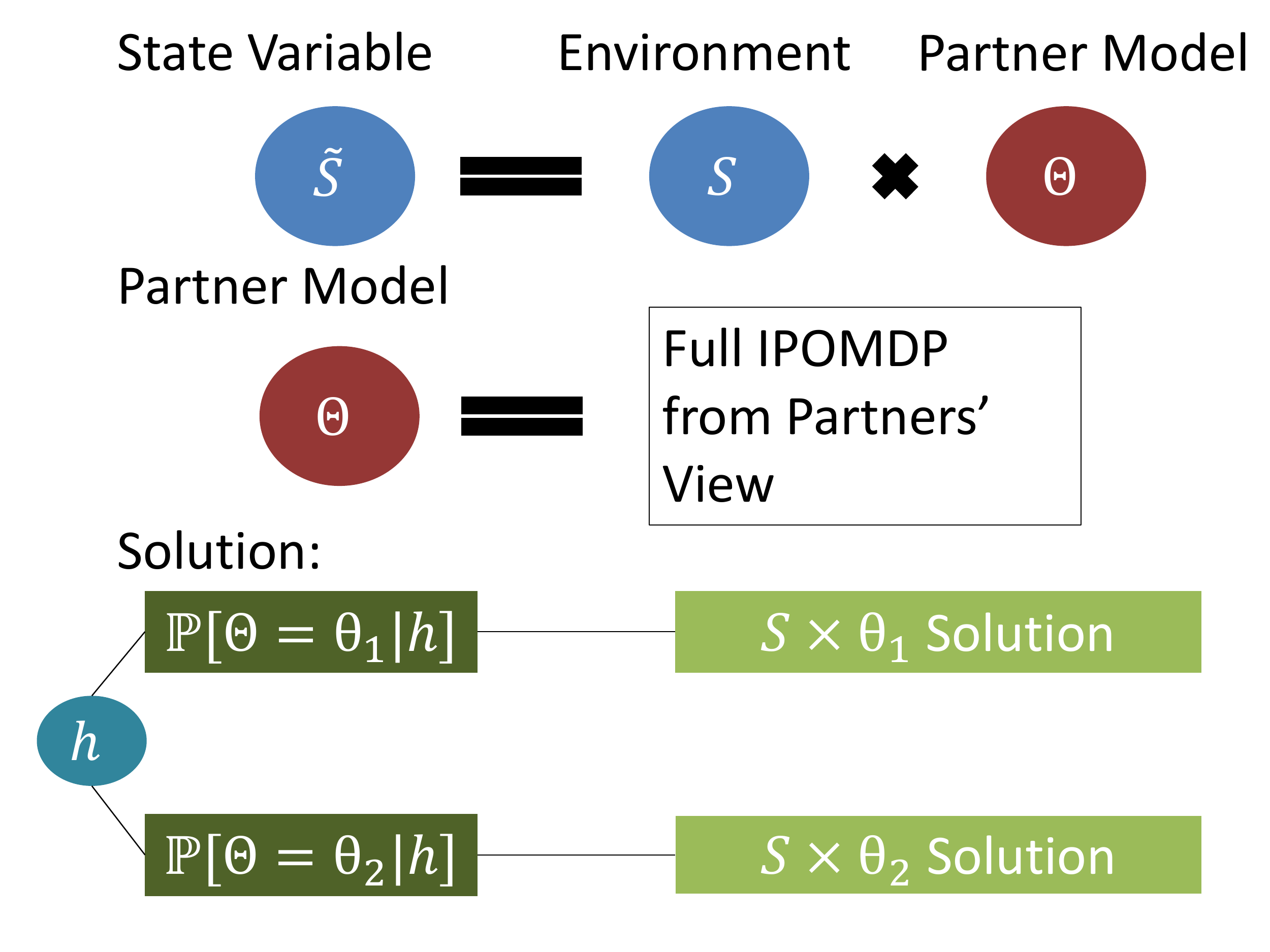}
\captionof{figure}{\small{Interactive Partially Observable Markov Decision Process.  Compared to a POMDP, the process is further 
complicated by the necessity to keep different models $\Theta$ of the other agents' intentions, so that evidence on the 
correct intentional model may be accrued in the belief state $\mathcal B(h)$. Any solution to the IPOMDP requires calculating all future action values depending on all possible future 
histories and the respective Belief states.}}
\label{fig:IPOMDPim}
\end{center}
\begin{definition}[Formal Definition - IPOMDP]
An IPOMDP is a collection of POMDPs such that the following holds:

Agents are indexed by the finite set $\mathcal I$.  Each agent $i
\in\mathcal I$ is described by a single POMDP ($\mathcal S^i$, $\mathcal
A^i$, $\mathcal O^i$, $\mathcal T^i$, $\mathcal W^i$, $\mathcal R^i$ ,$\Pi^i$,
$\mathcal B_0^i )$ denoting its actual decision making process. We first
define the physical state space $\mathcal S^i_{\text{\rm phys}}$: an
element $s \in S^i_{\text{\rm phys}}$  is a complete setting of all
features of the environment that determine the action possibilities
$\mathcal A^i$  and obtainable rewards $\mathcal R^i$ of $i$ for the present and all possible following
histories, from the point of view of $i$. The physical state space
$\mathcal S_{\text{\rm phys}}^i$ is augmented by the set $\mathcal D^i$
of models of the partner agents $\theta^{ij} \in \mathcal D^{i}, j \in \mathcal I
\setminus\{ i\}$, called intentional models, which are themselves POMDPs
$\theta^{ij}$=$(\mathcal S^{ij}$, $\mathcal A^{ij}$, $\mathcal O^{ij}$,
$\mathcal T^{ij}$, $\mathcal W^{ij}$, $\mathcal R^{ij}$, $\Pi^{ij}$, $\mathcal
B_0^{ij})$. These describe how agent $i$ believes agent $j$ perceives the world and reaches its
decisions.
 The  possible state space of agent $i$ can be written $\mathcal S^i = \mathcal{S}^i_{\text{\rm phys}}\times
\mathcal D^i$ and a given state can be written $\tilde{s}^i= (s^{i},
\times_j \theta^{ij}) $, where $s^i\in\mathcal S^i_{\text{\rm phys}}$ is
the physical state of the environment and $\theta^{ij}$ are the models
of the other agents. Note that the intentional models $\theta^{ij}$ contain 
themselves state spaces that encode the history of the game as observed by 
agent $j$ from the point of view of  agent $i$.  The elements of $\mathcal S^i$ are called interactive states.
Agents themselves act according to the softmax function of
history-action values, and assume that their interactive partner agents do the
same.
The elements of the definition are summarized in figure \ref{fig:IPOMDPim}.
\end{definition}
\begin{con}
We denote by capital $S$ and capital $\tilde{S}$ the random variables, that encode uncertainty about the 
physical state and the interactive state respectively.
\end{con}

When choosing the set of intentional models, we consider agents and
their partners to engage in a cognitive hierarchy of successive
mentalization steps \cite{Camerer, Costa}, depicted in figure $\ref{fig:cToM}$. The simplest agent can try to infer
what kind of partner it faces (level $0$ thinking). The next simplest
agent could additionally try to infer what the partner might be thinking
of it (level $1$). Next, the agent might try to understand their
partner's inferences about the agent's thinking about the partner (level
$2$). Generally, this would enable a potentially unbounded chain of
mentalization steps. It is a tenet of cognitive hierarchy theory
\cite{Camerer} that the hierarchy terminates  finitely and for 
many tasks 
after only very few steps (e.g., Poisson, with a mean of around $1.5$)  .

\begin{center}
\includegraphics[width =4.2 in, height = 2.2in]{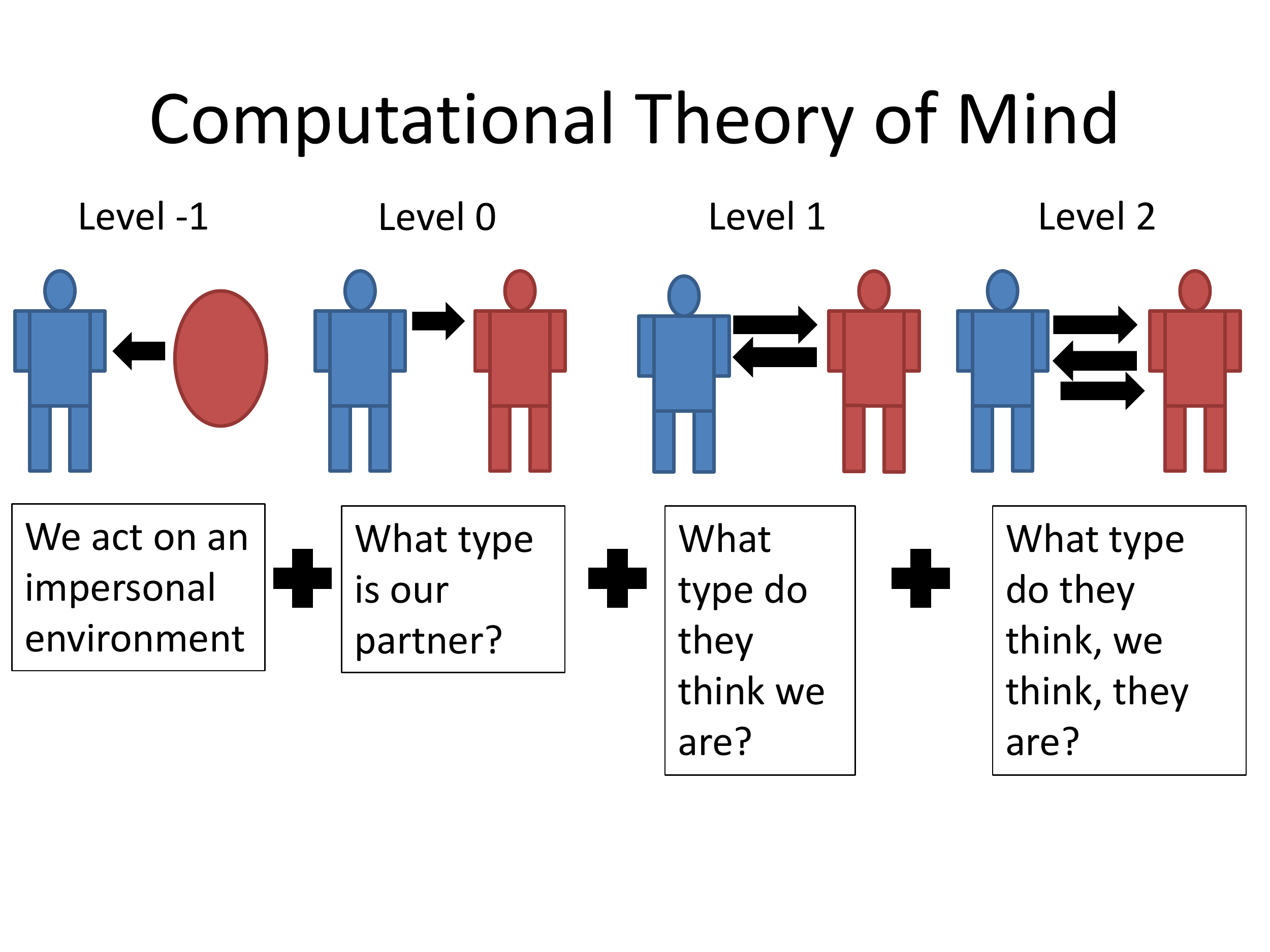}
\captionof{figure}{\small{Computational Theory of Mind (ToM) formalizes the notion of our understanding of other peoples 
thought processes.}} 
\label{fig:cToM}
\end{center}
We formalize this notion as follows.
\begin{definition}[A Hierarchy Of Intentional Models]
Since models of the partner agent may contain interactive states in which it in turn models
 the agent $i$, we can specify a hierarchical intentional structure $\mathcal D^{i,l}$,
 built from what we call the level $l
\geq -1$ intentional models $\mathcal D^{i, l}$. $\mathcal D^{i, l}$ is
defined inductively from
\[
\theta^{ij, -1}  \in \mathcal D^{i,-1} \Leftrightarrow S^{ij, -1} =  \mathcal S^{ij}_{\text{\rm phys}} \times \{ \emptyset \}.
\]
This means that any level $-1$ intentional model reacts
strictly to the environment, without holding any further intentional
models. 
The higher levels are obtained as
\[
\theta^{ij, l}  \in \mathcal D^{i, l} \Leftrightarrow \mathcal S^{ij,l} = \mathcal S^{ij}_{\text{\rm phys}}\times \mathcal D^{ij, l-1}. 
\]
Here $ \mathcal D^{ij, l-1}$ denotes the $l-1$ intentional models, that agent $i$ thinks agent $j$ might hold of the other players. 
These level $l-1$ intentional models arise by the same procedure applied to the level $-1$ models that agent $i$ thinks agent $j$ might hold.
\end{definition}

\begin{definition}[Theory of Mind (ToM) Level]\label{ToMConvention}
We follow a similar assumption as the so called $k$-level thinking (see \cite{Costa}),
 in that we assume that each agent operates at a particular level $l_i$ (called
the agent's theory of mind (ToM) level; and which it is assumed to know),
and models all partners as being at level $l_j = l_{i}-1$. 
\end{definition}

We chose definition \ref{ToMConvention} for comparability with earlier
work \cite{Debbs2008, Ting2012}.

\begin{con}
It is necessary to be able to calculate the belief state in every POMDP that is encountered. 
An agent updates its belief state in a Bayesian manner,  following an action $a_t^i$ and an observation $o_{t+1}^i$. 
This leads to a sequential update rule operating over the belief state $ \mathbb P [ \tilde{S}^i_t | h_t^i ]$ 
of a given agent $i$ at a given time $t$: 
\begin{equation}\label{Update}
\mathbb P [ \tilde{S}^i_{t+1} =\tilde{s}_1 | \{h_t^i , a_t^i, o_{t+1}^i \} ] =\eta\mathcal W (o^i_{t+1}, a^i_t,\tilde{s}_1)  
\sum_{\tilde{s} \in \mathcal S^i} \mathcal T (\tilde{s}_1 ,a_t^i,\tilde{s} ) \mathbb P [ \tilde{S}^i_t=\tilde{s}| h_t^i ].
\end{equation}
Here $\eta$ is a normalization constant associated with the joint distribution of transition and 
observation probability, conditional on $\tilde{s}$, $\tilde{s}_1, o_{t+1}^i$  and $a^i_t$.
 The observation $o^i_{t+1}$ in particular incorporates any results of the actions of the other agents, before the next action 
of the given agent. 

We note that the above rule applies recursively to every intentional model 
in the nested structure $\mathcal D^i$, as every POMDP has a separate belief state.
\end{con}
This is slightly different from \cite{IPOMDP} so that the above update is conventional for a POMDP.
\begin{con}[Expected Utility Maximisation]
The decision making rule in our IPOMDP treatment is based 
on expected utility as
encoded in the reward function. 
The explicit formula for the action
value $Q(a^i_t , h^i_t)$ under a softmax policy (equation~\ref{eq:softmax}) is:
\begin{equation}\label{Bellman}
Q(a_t^i , h_t^i) =  R( a_t^i , h_t^i) + \sum_{o_{t+1}^i \in \mathcal O} \mathbb P [o_{t+1}^i | \{ h_t^i, a_t^i \}] \sum_{b\in\mathcal A^i} \gamma^{(i)} H (t+1,t) Q( b , h^i_{t+1}| t)\mathbb P [b| h_{t+1}^i].
\end{equation}
Here $h_{t+1} =  \{ h^i_t , a_t^i, o_{t+1}^i\}$ and $Q( b , h^i_{t+1}| t)$ denotes the 
action value at $t+1$ with the survival function conditioned to reference time $t$. $\gamma^{(i)}$ is the 
discount factor of agent $i$, rather than the $i$-th power. 
This defines a recursive Bellman equation, with the value of taking action $a^i_t$
given history $h^i_t$ being the expected immediate reward  $ R( a_t^i, h_t^i)$ plus the
expected value of future actions conditional on $a_t^i$ and its possible
consequences $o_{t+1}^i$ discounted by  $\gamma^i $.

The belief state $\mathcal B(h^i_t)$ allows us to link $h^i_t$ to a
distribution of interactive states and use $\mathcal W$ to calculate
$\mathbb P[o^i_{t+1} | h^i_t, a^i_t]$, in particular including the
reactions of other agents to the actions of one agent.
We call the resulting policy the "solution" to the IPOMDP. 
\end{con} 
\subsection{Equilibria and IPOMDPs}
Our central interest is in the use of the IPOMDP to capture the
interaction amongst human agents with limited cognitive resources and
time for their exchanges. It has been noted in \cite{Camerer} that 
the distribution of subject levels favours rather low values
(e.g., Poisson, with a mean of around $1.5$). 
In the opposite limit, sufficient conditions are known in which taking
the cognitive hierarchy out to infinity for all involved agents allows for 
at least one Bayes-Nash equilibrium solution (part II, theorem II, p. $322$ of
Harsanyi \cite{Harsanyi}) and sufficient conditions have been shown in 
\cite{Nyarko}, given which a solution to the infinite hierarchy model 
can be approximated by the sequence of finite 
hierarchy model solutions.
 A discussion of a different condition
can be found in \cite{IPOMDPEquilibrium}; however, this condition
does assume a infinite time horizon in the interaction.
In general, as \cite{Camerer}, p.$868$ notes, it is not 
true that the infinite hierarchy solution will be a Nash equilibrium.
For the purposes of computational psychiatry, we find the very mismatches and limitations, 
that prevent subjects' strategies to evolve to a (Bayes)-Nash equilibrium in the given 
time frame, to be of particular interest.
 Therefore we restrict our attention to quantal response equilibrium like 
behaviours (\cite{quantal}) , based on potentially inconsistent initial beliefs 
by the involved agents with ultimately very limited cognitive resources and 
finite time exchanges.
\subsection{Applying POMCP to an IPOMDP}
An IPOMDP is a collection of POMDPs, so POMCP is, in principle, applicable
to each encountered POMDP.

 However, unlike
the examples in \cite{POMCP2010}, an IPOMDP contains
the intentional model POMDPs $\theta^{ij}$ as part of the state space,
and these themselves contain a rich structure of beliefs. So, the state
is sampled from the belief state at the root for agent $i$ is an $I$
tuple $(\hat{s}^i, \hat{\theta}^{i1},\ldots ,\hat{\theta}^{i(|\mathcal
I|-1)})$ of a physical state $\hat{s}^i$ and $(|\mathcal I|-1)$ POMDPs,
one for each partner.  
(This is also akin to the random instantiation of players in \cite{Harsanyi}).
 Since the $\hat{\theta}^{ij}$ still contain belief
states in their own right, it is still necessary to do some explicit
inference during the creation of each tree. Indeed, explicit inference
is hard to avoid altogether during simulation, as the interactive states
require the partner to be able to learn \cite{IPOMDP}.
Nevertheless, a number of performance improvements that we detail below
still allow us to apply the POMCP method involving substantial planning
horizons.

\subsection{Simplifications for dyadic repeated exchange}
Many social paradigms based upon game theory, including the iterated
ultimatum game, prisoners' dilemma, iterated "rock, paper, scissors"
(for $2$ agents) and the multi round trust game, involve repeated dyads.
In these, each interaction involves the same structure of physical states and
actions $(\mathcal S_{\textrm{phys}},\mathcal A)$ (see below), and 
all discount functions are $0$ past a finite horizon.
\begin{definition}[Dyadic Repeated Exchange without state uncertainty]\label{cert} 
  Consider a two agent IPOMDP framework in which there is no physical
  state uncertainty: both agents fully observe each others' actions and
  there is no uncertainty about environmental influence; and in which
  agents vary their play only based on intentional models and an agent does not believe 
that the partner can be made to transition between different intentional models by the agent's actions. Additionally, the
  framework is assumed to reset after each  exchange (i.e., after
  both agents have acted once).

 Formally this means: There is a fixed setting 
 $(\mathcal S_{\textrm{phys}},\mathcal A , \mathcal T)$,
such that physical states, actions from these states, transitions in the physical state
 and hence also obtainable rewards, differ only by a changing time 
index and there is no observational uncertainty and an agent does not believe 
that the partner can be made to transition between different intentional models by the agent's actions. Then after each exchange the framework is assumed to reset 
to the same distribution of physical
initial states $S$ within this setting (i.e. the game begins anew). 
\end{definition}  
Games of this sort admit an immediate simplification:
\begin{theorem}[Level $0$ Recombining Tree]\label{recomb}
In the situation of definition \ref{cert}, level $0$ action values at any given time only depend on the total set of 
actions and observations so far and not the order in which those exchanges were observed.
\end{theorem} 
\begin{proof}
The level $-1$ partner model only acts on the physical state it encounters 
and the physical state space variable $S$ is 
reset at the beginning of each round in the situation of $\ref{cert}$.
Therefore, given a state $s$ in the current round and 
an action $a$ by a level $0$ agent, 
the likelihood of each transition to some state $s_1$,  
$\mathcal T(s_1, a, s)$, and of making  
observation $o$, $\mathcal W (o, a, s_1)$,  is the same at every 
round from the point of view of the level $0$ agent. It follows that the 
cumulative belief update from equation \ref{Update}, from the 
initial beliefs $\mathcal B_0$ to the current beliefs, will not depend on 
the order in which the action observation pairs $(a,o)$ were observed.\qedhere
\end{proof}
This means, that depending on the size of the state space and the 
depth of planning of interest, we may analytically calculate level $0$ action values even online or use precalculated values for larger problems. 
Furthermore, because their action values will only depend on past exchanges and not on the order in which they were observed, 
their decision making tree can be reformulated as a recombining tree.

Sometimes, an additional simplification can be made:
\begin{theorem}[Trivialised Planning]\label{trivpartner}
  In the situation of definition \ref{cert}, if the two agents do not
  act simultaneously and the state transition of the second agent is
  entirely dependent on the action executed by the first agent (as in
  the multi round trust task); and additionally the intentional model of the partner 
can not be changed through the actions of the second agent, then a level $0$ second agent can gain no
  advantage from planning ahead, since their actions will not change the
  action choices of the first agent.
\end{theorem}
\begin{proof}
In the scenario described in theorem \ref{trivpartner} the physical state variable $S$ of the agent $2$ is entirely dependent on the action of the other agent. 
If the agent is level $0$, they model their partner as level $-1$ and by additional assumption the second agent does not believe 
that the partner can be made to transition between different intentional models by the second agent's actions, hence their partner will not change their distribution of state transitions, 
depending on the agents' actions and hence also their distribution of future obtainable rewards will not change. 
 \qedhere
\end{proof}
\begin{theorem}[Trivialised Theory of Mind Levels]\label{trivlevel}
In the situation of theorem \ref{trivpartner}, we state that for the first to go agent 
only the even theory of mind levels $k\in \{0 \} \cup 2 \mathbb N$ show distinct 
behaviours, while the odd levels $k\in 2 \mathbb N -1$ behave like one level below, meaning $k-1$. 
For the second to go partner equivalently, only the odd levels $k \in \{0 \} \cup 2\mathbb N -1$ show 
distinct behaviours.
\end{theorem}
\begin{proof}
In the scenario described in theorem \ref{trivpartner}, the second to go level $0$ agent behaves like a level $-1$ agent, 
as it does not benefit from modeling the partner. This implies that the first to go agent, gains no additional  
information at the level $1$ thinking, since the partner behaves like level $-1$, which was 
modeled by the level $0$ first to agent already. In turn, the level $2$ second to go agent
gains no additional information over the level $1$ second to go agent, as the their partner model 
does not change between modeling the partner at level $0$ or level $-1$.
By induction, we get the result. \qedhere
\end{proof}
Examples of the additional simplifications in theorems \ref{trivpartner}
and \ref{trivlevel} can be seen in the ultimatum game and the multi round trust
game.
\subsection{The Trust Task}\label{TruPlan}
\begin{center}
\includegraphics[width=4in, height = 2in]{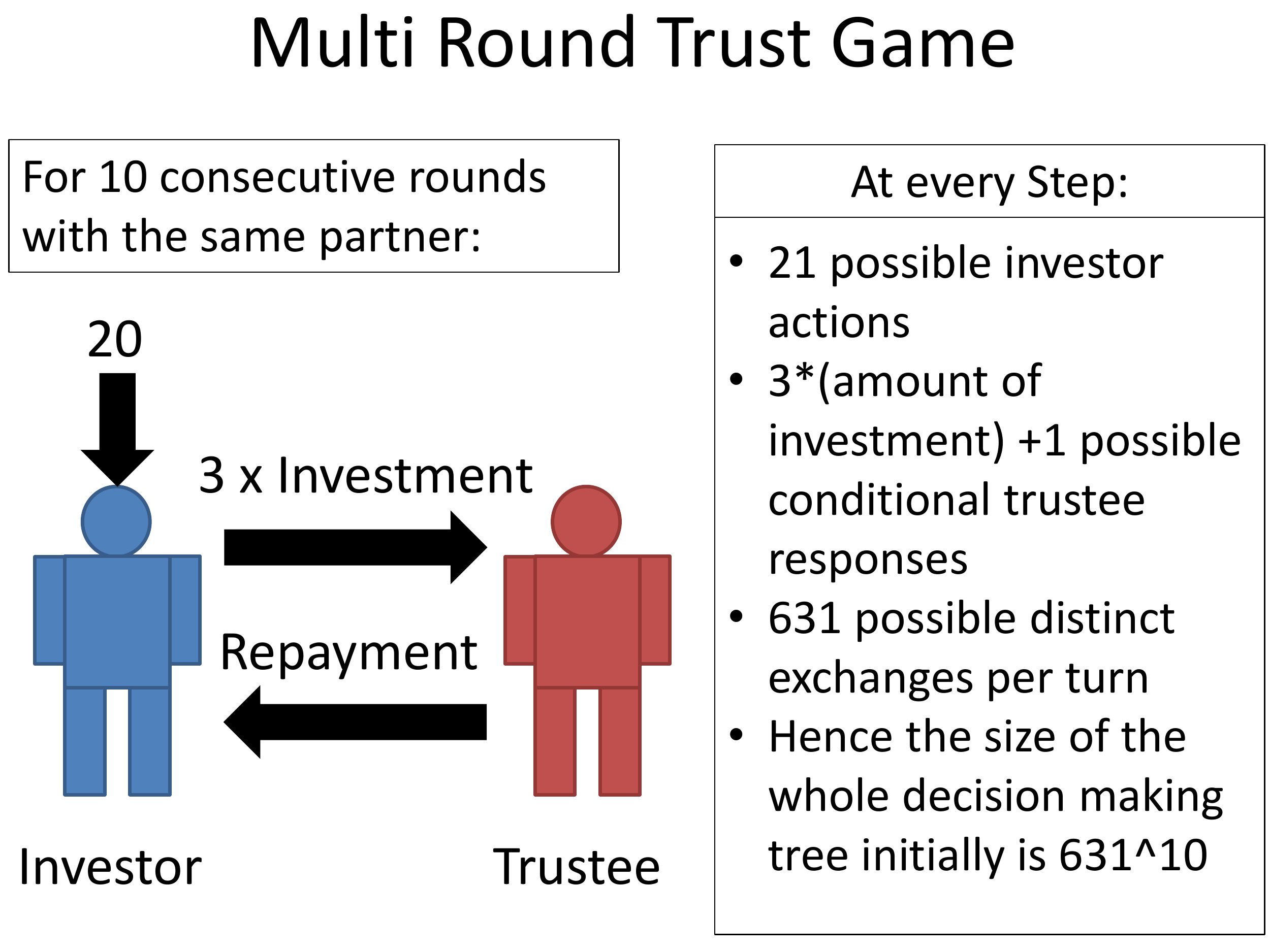}
	\captionof{figure}{\small{ Physical features of the multi round trust game.}
	\label{fig:trusttask}}
\end{center}
The multi-round trust task, illustrated in figure~\ref{fig:trusttask} is
a paradigm social exchange game. It involves two people, one playing
the role of an 'investor' the other the one of a 'trustee', over $10$ sequential rounds, expressed 
by a time index $t = 1, 2, \ldots, 10$. Both agents
know all the rules of the game.  In each round, the investor receives an
initial endowment of $20$ monetary units. The investor can send any of this
amount to the trustee. The experimenter trebles this quantity and then
the trustee decides how much to send back to the investor, between $0$
points and the whole amount that she receives. The repayment by the
trustee is not increased by the experimenter. After the trustee's
action, the investor is informed, and the next round starts. 
We consider the trust task as an IPOMDP with two agents, i.e., $\mathcal
I = \{ I , T \}$ contains just $I$ for the investor and $T$ for the
trustee. We consider the state to contain two components; one physical
and observable (the endowment and investments), the other non-physical
and non-observable (in our case, parameters of the utility function). It
is the latter that leads to the partial observability in the IPOMDP.
Following \cite{Debbs2008}, we reduce complexity by quantizing the
actions and the (non-observable) states of both investor and trustee --
shown for one complete round in figure ~\ref{fig:trustdisc}.
The actions are quantized into $5$ fractional categories shown in figure 
~\ref{fig:trustdisc}.
 For the
investor, we consider $a^I\in \{ 0, 0.25, 0.5, 0.75, 1 \}$
(corresponding to an investment of \$20$\times a^I$, and encompassing even
investment ranges). For the trustee, we consider $a^T\in \{0, 0.167,
0.333, 0.5, 0.67 \}$ (corresponding to a return of \$3$\times$20$\times a^Ia^T$, and
encompassing even return ranges). Note that the trustee's action is
degenerate if the investor gives $0$. The pure
monetary payoffs for both agents in each round are
\[
\mathrm{investor:} \chi^I (a^I, a^T) = 20 - 20 a^I + 3 \times 20 a^I
a^T. 
\]
\[
\mathrm{trustee:} \chi^T(a^I, a^T) = 3 \times 20 a^I -   3 \times 20 a^I a^T.
\]
The payoffs of all possible combinations and both partners are depicted
in figure~\ref{fig:trustpay}.  
In IPOMDP terms, the investor's physical state is static, whereas the
trustee's state space is conditional on the previous action of the
investor. The investor's possible observations are the trustees
responses, with a likelihood that depends entirely on the investor's
intentional model of the trustee. The trustee observes the 
investor's action, which also determines
the trustee's new physical state, as shown in figure ~\ref{fig:truststate}.
\begin{center}
\vbox{\hbox{\includegraphics[width=3in, height = 2in]{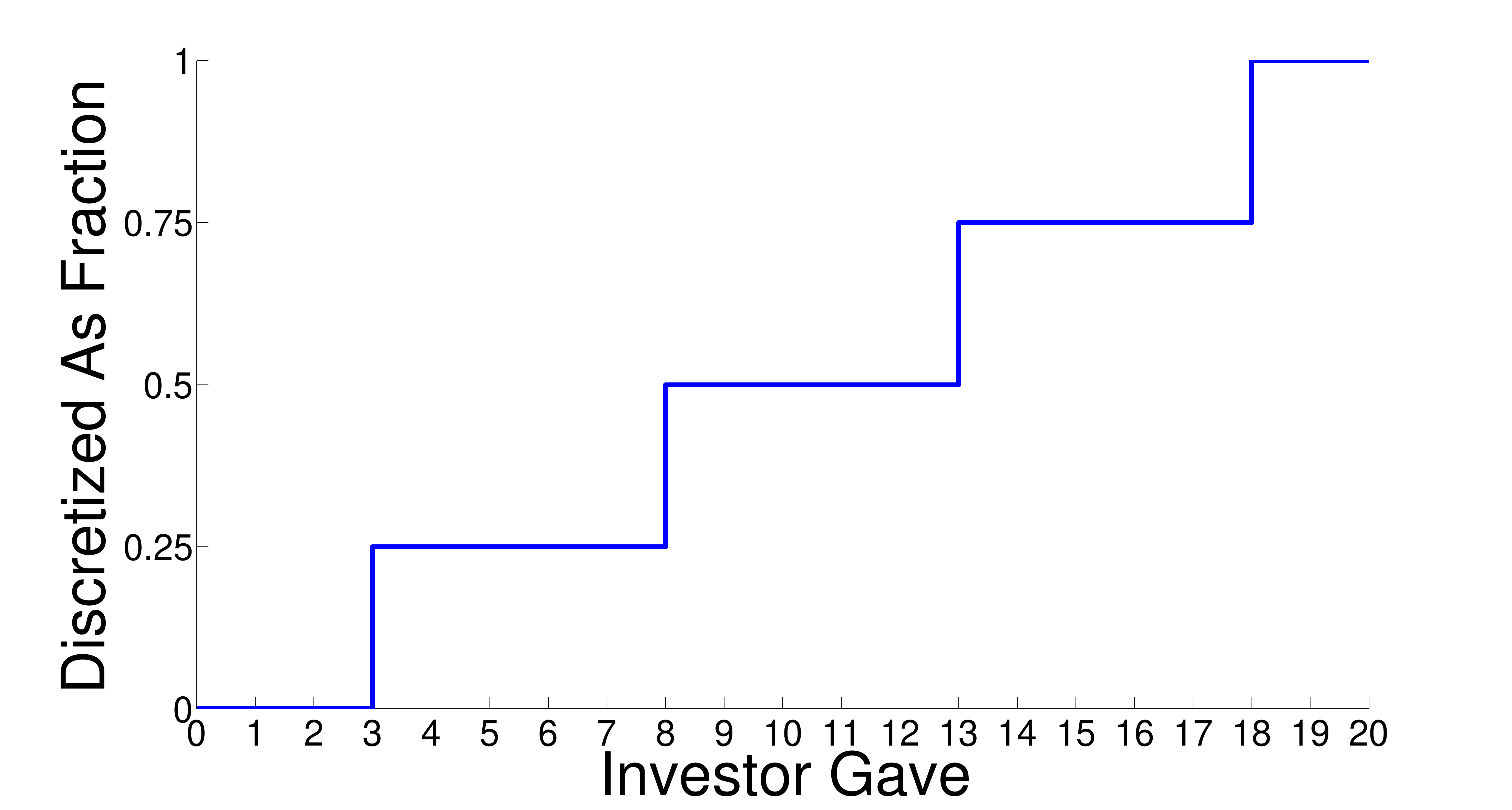}\includegraphics[width=3in, height = 2in]{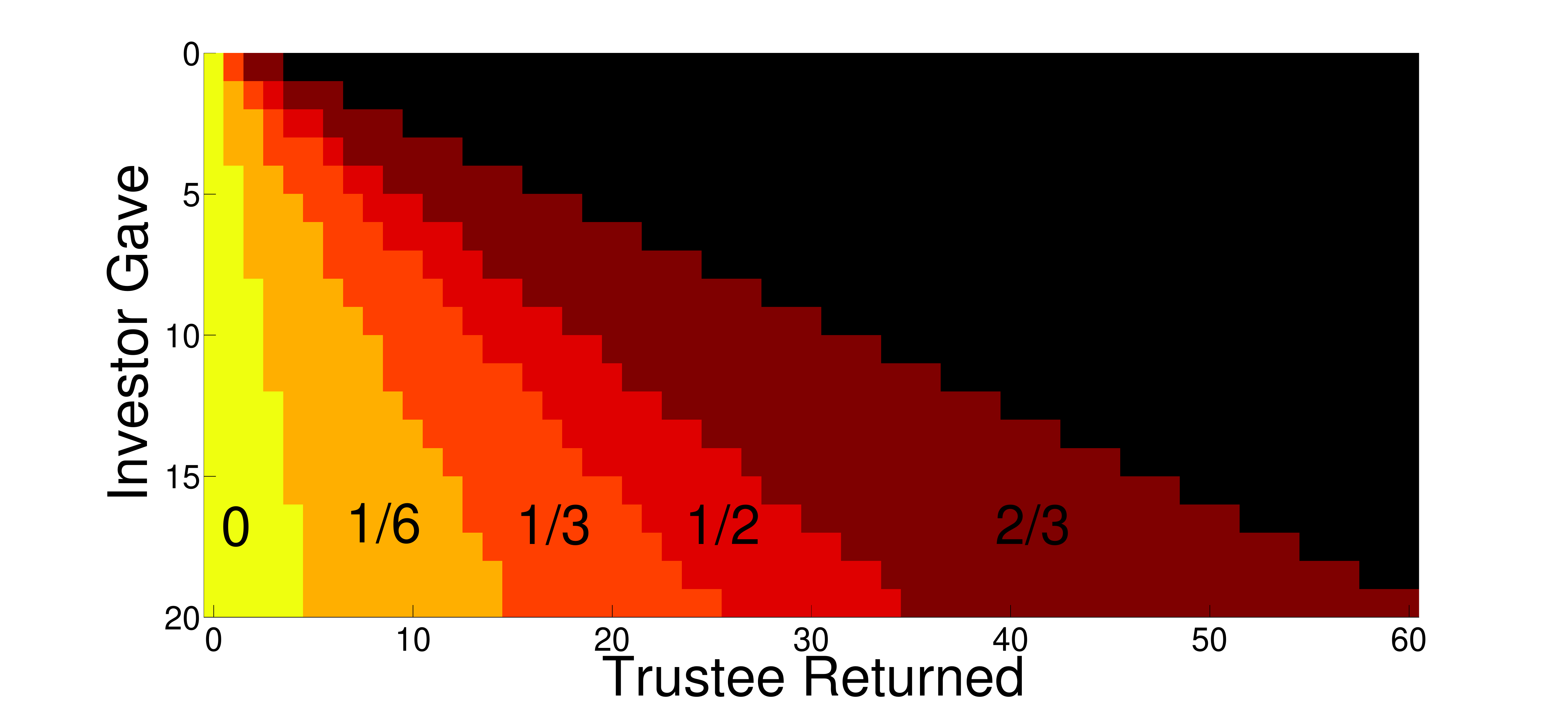}}}
	\captionof{figure}{\small{ Investor: (left) The $21$ possible actions are summarized into $5$ possible investment
 categories. Trustee: ( right) returns are classified into $5$ possible categories, conditionally on investor action. Impossible returns are marked in black.  }
\label{fig:trustdisc}}
\end{center}
\begin{center}
 \includegraphics[width=6in, height = 2in]{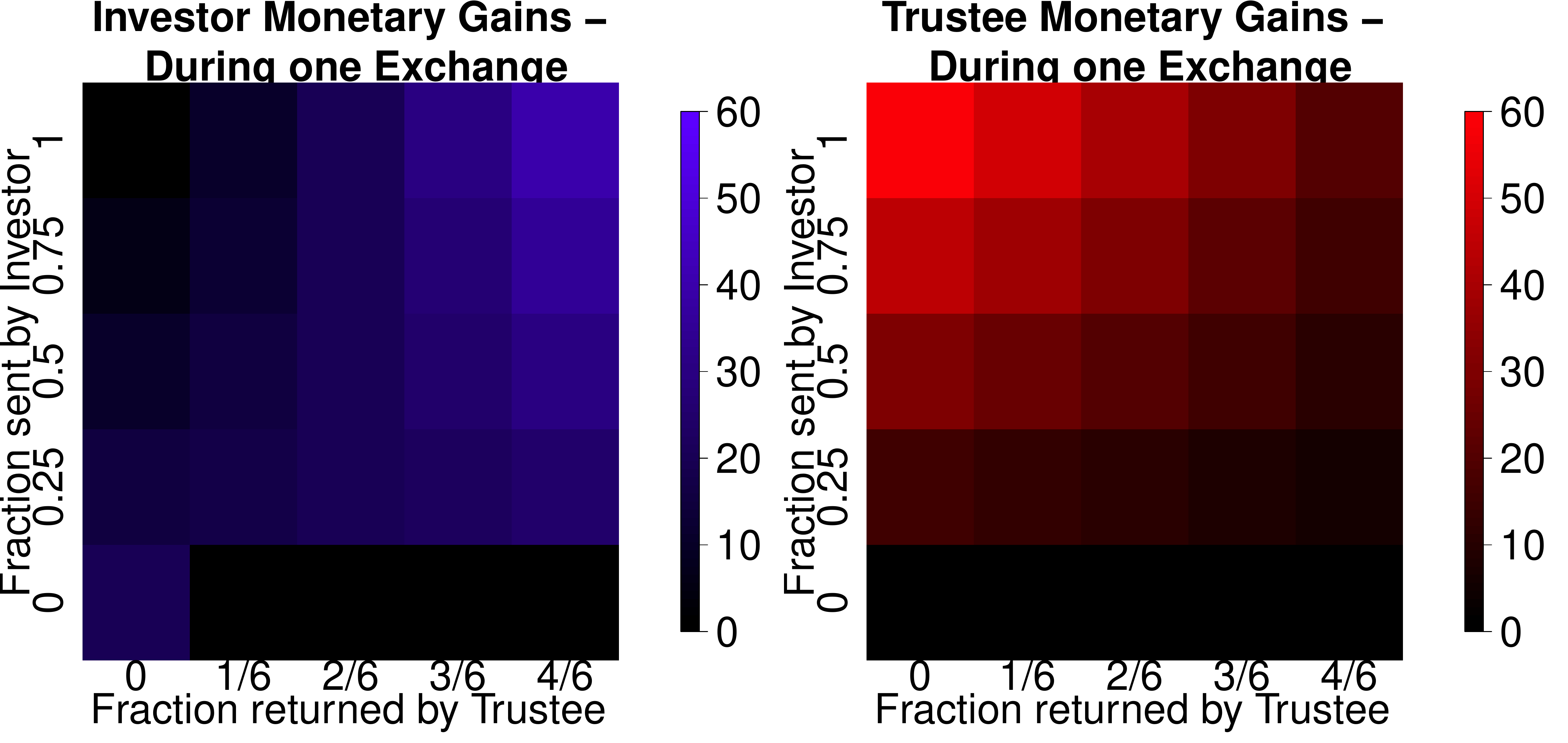} 
	\captionof{figure}{ \small{Payoffs in the multi round trust task. (left) Investor payoffs for an single exchange. (right) Trustee payoffs for an single exchange.}
	\label{fig:trustpay}}
\end{center}
\begin{center}
\vbox{\hbox{\includegraphics[width=3in, height = 2in]{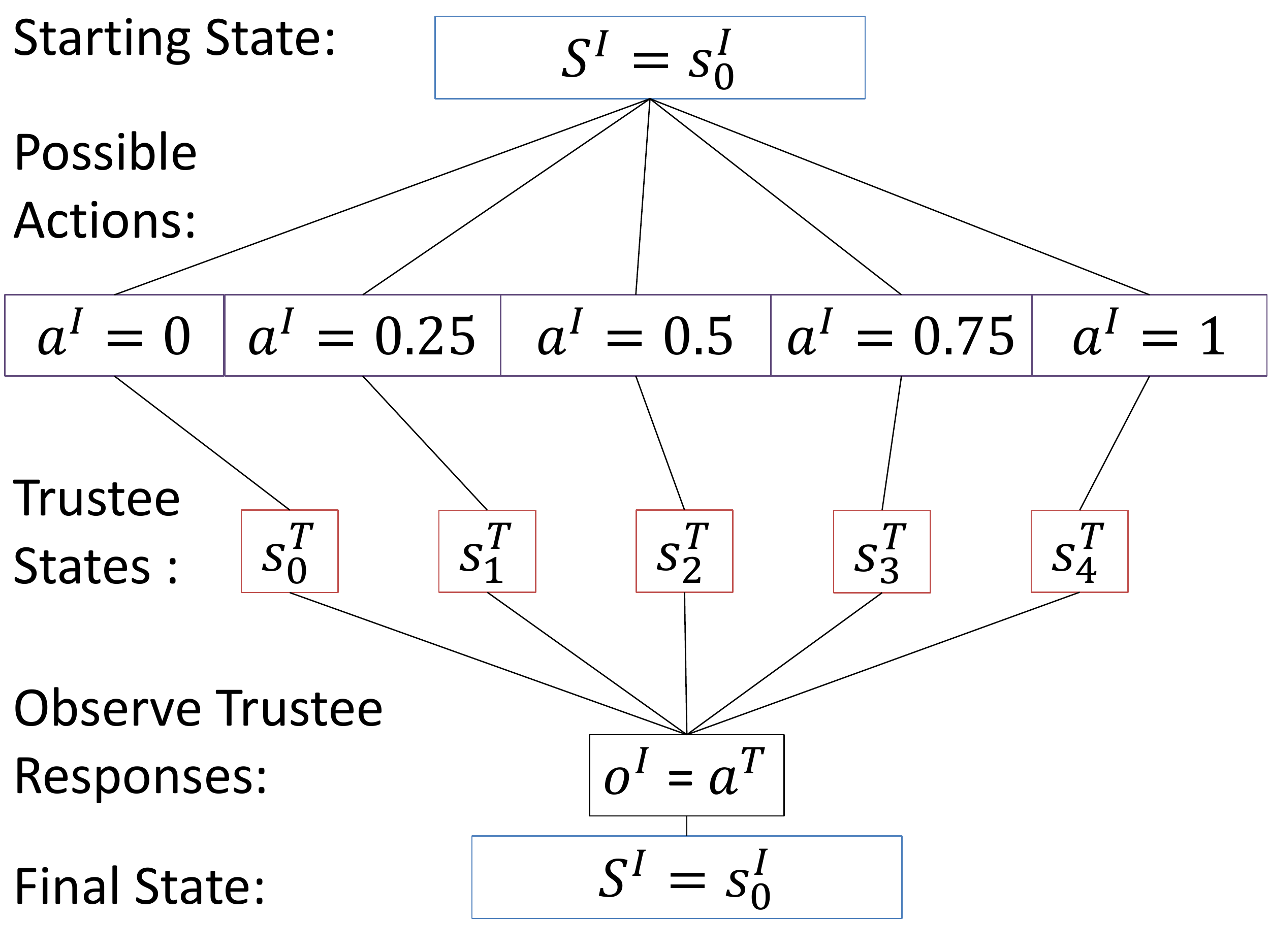} \includegraphics[width=3in, height = 2in]{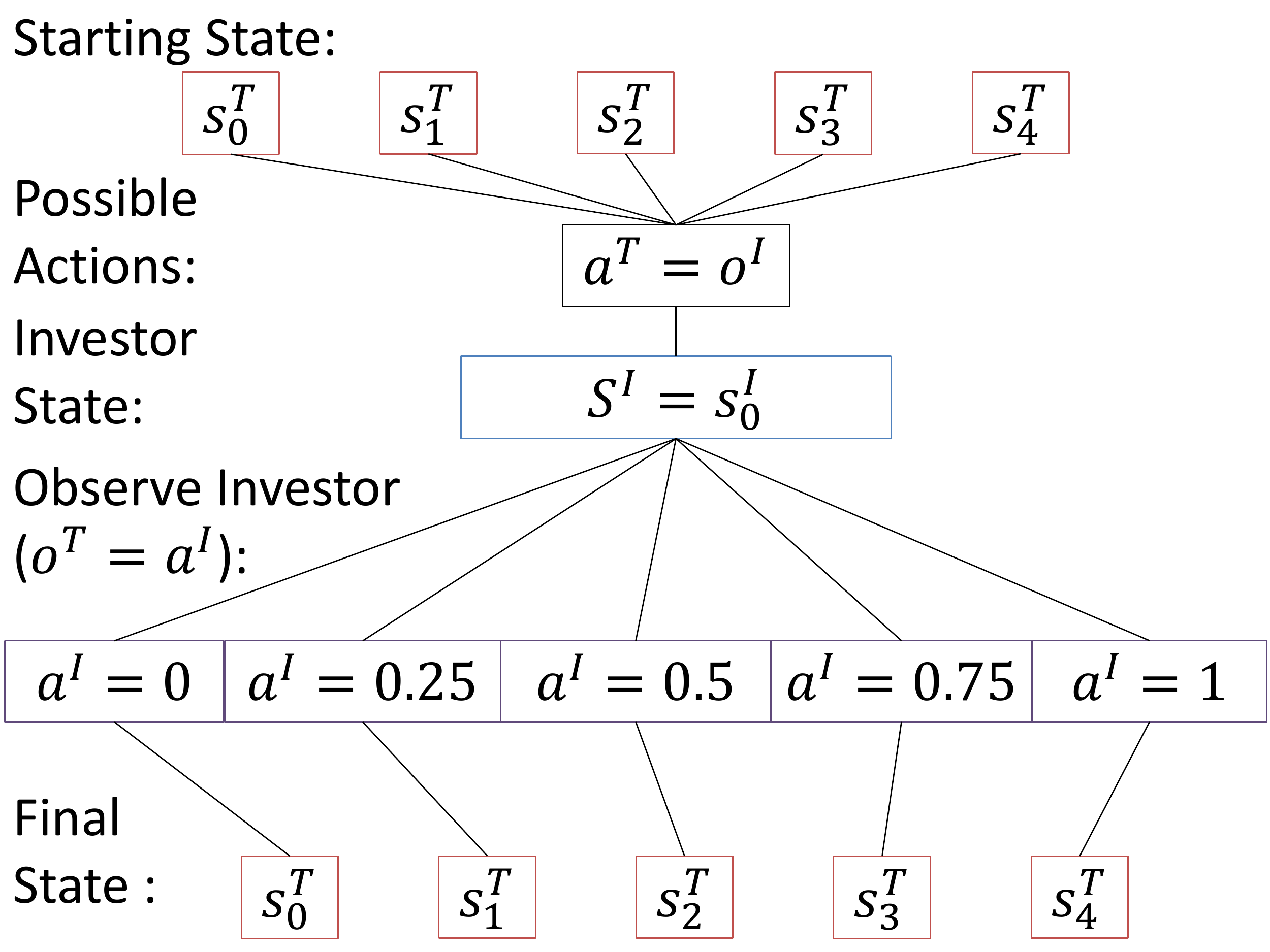} }}
	\captionof{figure}{\small{(Physical) Transitions and Observations: (left) Physical state transitions and observations of the investor. 
The trustee's actions are summarized to $a^T$, as they can not change the following physical state transition. (right) Physical state transitions and observations of the trustee.
The trustee's actions are summarized to $a^T$, as they can not change the following physical state transition.}
\label{fig:truststate}}
\end{center}
\subsubsection{Inequality Aversion - Compulsion to Fairness}
The aspects of the states of investor and trustee that induce partial observability are assumed to arise from differential levels of cooperation.
\begin{center}

\includegraphics[width=6in, height = 6in]{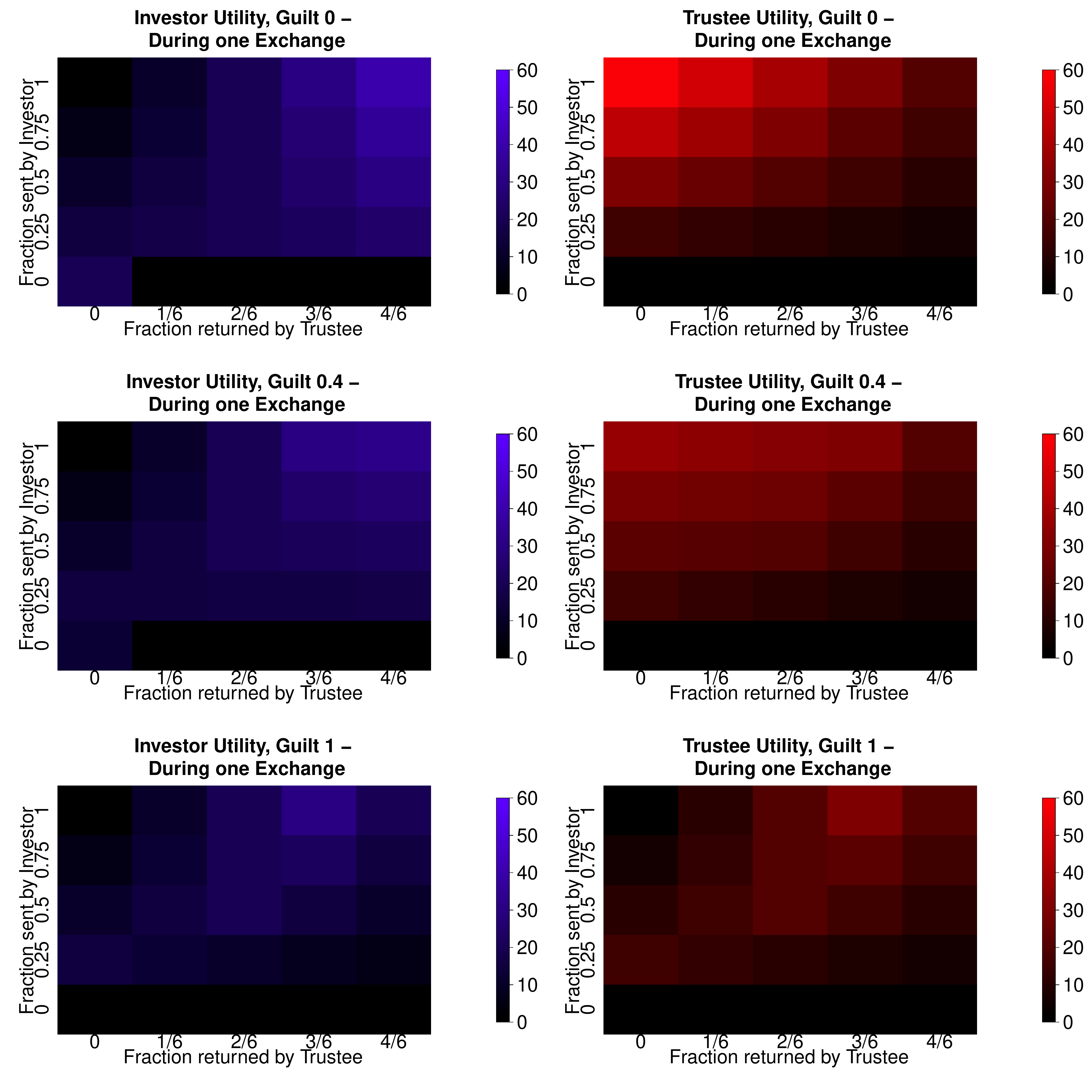} 
	\captionof{figure}{\small{Immediate Fehr-Schmidt utilities for a single exchange \cite{Fehr}. Left column shows investor preferences:  (top left) 
Completely unguilty investor values only the immediate payoff, (middle left) Guilt $0.4$ investor is less likely to keep everything to 
themselves (bottom left corner option), (bottom left) Guilt $1$ investor will never keep everything to themselves (bottom left option). 
Right column shows trustee preferences: (top right) unguilty trusty would like to keep everything to themselves. (middle right) 
Guilt $0.4$ is more likely to return at least a fraction of the gains. (bottom right) Guilt $1$ trustee will strife to return the fair split always.}
	\label{fig:FSUtil}}
\end{center}
One convenient (though not unique) way to characterize this is via the Fehr-Schmidt inequality aversion utility function (figure $\ref{fig:FSUtil}$). This 
allows us to account for the observation that many trustees 
 return an even split even on the last exchange of the $10$ rounds, even though no further gain is possible. 
We make no claim that this is the only explanation for such behaviour, 
but it is a tractable and well-established mechanism that has been used
successfully in other tasks (\cite{Norm,  Fehr, FehrCamerer} ).  
  For the investor, this suggests that:
\begin{equation}\label{FS}
r^I (a^I, a^T, \alpha^I) =\chi^I(a^I,a^T) -\alpha^I \max \{ \chi^I(a^I,a^T) -\chi^T (a^I, a^T),0\}.
\end{equation}
Here, $\alpha^I$ is called the "guilt" parameter of the investor and
quantifies their aversion to unequal outcomes in their favor. We quantize guilt into
$3$ concrete guilt types $\{0,0.4,1\} = \{ \alpha_1,\alpha_2, \alpha_3 \}$. 
Similarly, the trustee's utility is 
\begin{equation}\label{FST}
r^T(a^I, a^T, \alpha^T) = \chi^T(a^I,a^T) -\alpha^T \max \{ \chi^T(a^I,a^T)-\chi^I(a^I,a^T) , 0\},
\end{equation}
with the same possible guilt types. We choose these particular values, 
as guilt values above $0.5$ tend to produce similar behaviours as $\alpha =1$ and the values below $0.3$ tend to 
behave very similar to $\alpha =0$. Thus we take $\alpha_1$ to represent guilt values in $[0, 0.3]$, $\alpha_2$ to represent 
guilt values in $(0.3, 0.5)$ and $\alpha_2$ to represent guilt values in $[0.5,1]$.  
We assume that neither agents' actual guilt type changes during the $10$
exchanges.
\subsubsection{Planning Behaviour}
The survival functions $H^I$ and $H^T$ are
used to delimit the planning horizon. The agents are required not to 
plan beyond the end of the game at time $10$ and within that constraint they 
are supposed to plan $P$ steps ahead into the interaction. This results in the 
following form for the survival functions (regardless whether for 
investor or trustee):
\begin{equation}
H_P (\tau, t) =1, \quad (\tau-t) \leq P \land (\tau+t) \leq 10, \qquad
H_P(\tau,t) =0, \quad (\tau-t) > P \lor (\tau+t) > 10. 
\end{equation}
The value $P$ is called the planning horizon. We consider $P\in \{
0,2,7\}$ for immediate, medium and long planning types. We chose these values
as $P=7$ covers the range of behaviours from 
$P=4$ to $P=9$, while planning $2$ yields compatibility to earlier works (\cite{Debbs2008, Ting2012}) and 
allows to have short planning but high level agents, covering the range of behaviours for planning $P=1$ to $P=3$. We confirm later that the 
behaviour of $P=7$ and $P=9$ agents is almost identical; and the former saves memory 
and processing time. Agents are characterized to assume their opponents have the same degree of 
planning as they do.
The discounting factors $\gamma^I$ and $\gamma^T$ are set to $1$ in our setting.
\subsection{Belief State}
Since all agents use 
their own planning horizon in modeling the partner and level $k$ agents model their partner at level $k-1$, inference in 
intentional models in this analysis is restricted to the guilt parameter $\alpha$.  Using a 
categorical distribution on the guilt parameter and Dirichlet prior on the probabilities 
of the categorical distribution, we get a Dirichlet-Multinomial distribution for the probabilities of 
an agent having a given guilt type at some point during the exchange.
Hence $\mathcal B_0$ is a Dirichlet-Multinomial distribution
,
\[
\mathcal B_0  \sim DirMult(a_0), \qquad a_0 = (1,1,1)
\]
with the initial belief state
\[
\mathbb P [\alpha^{\textrm{partner}}=\alpha_i| h = \emptyset] = \frac{1}{3}.
\]
Keeping consistent with the model in ~\cite{Ting2012}, our approximation of the posterior distribution is also a Dirichlet-Multinomial distribution 
with the parameters of the Dirichlet prior being updated to
\begin{equation}\label{Belief}
a_{t+1}^{i}=a_t^{i} +\mathbb P[o_{t+1}=\textrm{observed action}| \alpha^{\textrm{partner}}=\alpha_i].
\end{equation}
writing $\alpha^{\textrm{partner}}$ for the intentional models.

\subsubsection{Theory of Mind Levels and Agent Characterization}\label{ToMActual}
Since the physical state transition of the trustee is fully dependent on the investors' action
and one agents' guilt type can not be changed by the actions of the other agent,  
theorem \ref{trivpartner} implies that the level $0$ trustee is trivial, gaining nothing from planning ahead. 
Conversely, the level $0$ investor can use a recombining tree as in theorem \ref{recomb}. Therefore, the chain of 
cognitive hierarchy steps for the investor is $l^I \in\{ 0 \} \cup  \{ 2n | n\in\mathbb N  \}$, and for the trustee, it is $l^T\in\{ 0 \} \cup \{ 2n-1 | n\in\mathbb N \} $. 
Trustee planning is trivial until the trustee does at least reach theory of mind level $1$.
Assuming $\beta = \frac 1 3$ in $\ref{eq:softmax}$, determined empirically from real subject data \cite{Ting2012} for suitably noisy behaviour, 
our subjects are then characterized via the triplet  $(k ,\alpha , P)$ of theory of mind level $k$, guilt parameter $\alpha \in \{0, 0.4, 1 \}$ and 
planning horizon $P\in\{ 0,2,7 \}$.

\subsection{Level $-1$ and POMCP rollout mechanism}

The level $-1$ models are obtained by having the level $-1$ agent always assume all partner 
types to be equally likely ($\mathbb P[\alpha^{\textrm{partner}}=\alpha_i  ]= \frac 1 3, \forall i$), 
setting the planning horizon to $0$, meaning the partner acts on immediate 
utilities only, and calculating the agent's expected utilities after marginalizing over partner types 
and their respective response probabilities based on their immediate utilities.

In the POMCP treatment of the multi round trust game, if a simulated agent reaches 
a given history for the first time, a value estimate for the new node is 
derived by treating the agent as level $-1$ 
and using an $\epsilon$-greedy decision making mechanism on the expected utilities  to determine their actions until the present planning horizon.

\section{Results}

We adapted the POMCP algorithm \cite{POMCP2010} to solve IPOMDPs
\cite{IPOMDP}, and cast the multi-round trust task as an IPOMDP that
could thus be solved. We made a number of approximations that were
prefigured in past work in this domain \cite{Debbs2008, Ting2012}, and also made
various observations that dramatically simplified the task of planning,
without altering the formal solutions. This allowed us to look at longer
planning horizons, which is important for the full power of the
intentional modeling to become clear.

Here, we first seek to use this new and more powerful planning method to
understand the classes of behaviour that arise from different settings
of the parameters in section \ref{ParadigmBehaviours}. From the study of human interactions \cite{Brooks2008}, the importance of coaxing (returning more than the fair split) has been established.
From our own study of the data collected so far, we define four coarse types of 'pure'
interactions, which we call
"Cooperation", "Coaxing to Cooperation", "Coaxing to Exploitation", "Greedy" ; we
conceptualize how these might arise. We also delimit the potential consequences of
having overly restricted the planning horizon in past work in this
domain, and examine the qualitative interactive signatures (such as how
quickly average investments and repayments rise or fall) that might
best capture the characteristics of human subjects playing the game.

We then continue to discuss the quality of statistical inference, 
by carrying out model inversion for our new method in 
section \ref{ConfusionDiscussion} and comparing to earlier work in 
this domain \cite{Debbs2008}.

Finally, we treat real subject data collected for an earlier study (\cite{Ting2012}) 
in section \ref{EarlierWork} 
and show that our new approach recovers significant behavioural differences 
not obtained by earlier models and offers a significant improvement in the 
classification of subject behaviour through the inclusion of the planning 
parameter in the estimation and the quality of estimation on the 
trustee side. 

The materials used in this section, as well as the code used to generate them, 
can be found on Andreas Hula's \href{https://github.com/AndreasHula/Trust}{github  repository}. 
All material was generated on the local WTCN cluster. 
We used R \cite{RStat} and Matlab \cite{MATLAB:2010} for data analysis 
and the boost C++ libraries \cite{BoostSite} for code generation. 

\subsection{Modalities}

All simulations were run on the local cluster at the Wellcome Trust
Centre for Neuroimaging.  For sample paths and posterior distributions,
for each pairing of investor guilt, investor sophistication and trustee
guilt and trustee sophistication, $60$ full games of $10$ exchanges
each were simulated, totaling $8100$ games. Additionally, in order to validate the
estimation, a uniform mix of all parameters was used, implying a total 
of $2025$ full games.

To reduce the variance of the estimation, we employed a pre-search
method. Agents with ToM greater than $0$ first explored the constant
strategies (offering/returning a fixed fraction) to obtain a minimal set
of $\tilde{Q}$ values from which to start searching for the optimal
policy using {\sc SoftUCT}. This ensures that inference will not "get
stuck" in a close-to-optimal initial offer just because another initial
offer was not adequately explored. This is more specific than just
increasing the exploration bonus in the {\sc SoftUCT} rule, which would
diffuse the search during all stages, rather than helping search from a
stable initial grid.

We set a number $n$ of simulations for the initial step, where the beliefs about the partner are still uniform and the time horizon is still furthest away. We then reduce the number of simulations as the time horizon approaches $(n, n \frac 9 {10}, n\frac 8 {10}, \ldots, n\frac 1 {10})$ .

\subsection{Simulation And Statistical Inference}
Unless stated otherwise, we employ an inverse temperature in the softmax
of $\beta =\frac 1 3$ (noting the substantial scale of the rewards). The
exploration constant for POMCP was set to $c = 25$. The initial
beliefs were uniform $a_i = 1, \forall i$, for each
subject. For the $3$ possible guilt types we use the following
expression while in text: $\alpha = \alpha_1$ is "greedy",
$\alpha = \alpha_2$ is "pragmatic" and $\alpha = \alpha_3$ is
"guilty". However, on all the graphs, we give the exact model
classification in the form $I: (k^I, \alpha^I, P^I)$ for the investor
and $T:(k^T, \alpha^T, P^T)$ for the trustee. 

We
present average results over multiple runs generated stochastically from
each setting of the parameter values. In the figures, we report the
\emph{actual}
characteristics of investor and trustee; however, in
keeping with the overall model, although each agent knows their own
parameters, they are each inferring their opponents' degree of guilt
based on their initial priors.

As a consequence of the observation in section~\ref{ToMActual}, we only
consider $k\in\{0,2\}$ for the investor and $k\in\{0,1 \}$ for the
trustee. Planning horizons are restricted to $P\in\{0,2,7\}$, as noted
in section~$\ref{TruPlan}$, with the level $0$ trustee always having a
planning horizon of $0$.

Actions for both agents are parametrized as in
section~\ref{TruPlan} and averaged across identical parameter
pairings. In the graphs, we show actions in terms of the percentages of
the available points that are offered or returned. For the investor, the
numerical amounts can be read directly from the graphs; for the trustee,
these amounts depend on the investor's action. In the figures, we report
the \emph{actual} characteristics of investor and trustee; however, in
keeping with the overall model, although each player knows their own
parameters, they are each inferring their opponents' degree of guilt
based on their initial priors.

Dual to generating behaviour from the model is to invert it to find
parameter settings that best explain observed interactions
\cite{Debbs2008,Ting2012}. Conceptually, this can be done by simulating
exchanges between partners of given parameter settings $(k, \alpha, P)$,
taking the observed history of investments and responses, and using a
maximum likelihood estimation procedure which finds the settings for
both agents that maximise the chance that simulated exchanges between
agents possessing those values would match the actual, observed
exchange.  We calculate the action likelihoods through the POMCP method
outlined in the earlier section \ref{POMCP} and accumulate the negative log
likelihoods, looking for the combination that produces the smallest
negative loglikelihood. This is carried out for each combination of
guilt and sophistication for both investor and trustee.

\subsection{Paradigmatic Behaviours}\label{ParadigmBehaviours}

Figures~\ref{fig:cooperation} (with the additional outcome comparison in
figure~\ref{fig:gains}), and figures~\ref{fig:coaxexpl}
and~\ref{fig:coaxcoop} show the three characteristic types of behaviour,
in each case for two sets of parameters for investor and
trustee. The upper graphs show the average histories of actions of the
investor ( blue) and trustee (red) across the $10$ rounds; the middle
graphs show the mean posterior distributions over the three guilt
parameters ($0, 0.4, 1$) as estimated by the investor and the 
lower graphs show the mean posterior distribution by the trustee
(right) at four stages in the game (rounds 0, 3, 6 and 9). These show
how well the agents of each type are making inferences about their
partners.

Figure~\ref{fig:cooperation} shows evidence for strong cooperation
between two agents who are characterized by high inequity aversion (i.e., guilty).
Cooperation develops more slowly for agents with shorter (left) than
longer (right) planning horizons, enabling a reliable distinction
between different guilty pairs.  This is shown more explicitly in
figure~\ref{fig:gains} in terms of the total amount of money made by
both participants.  Both cases can be seen as cases of a tit 
for tat like approach by the players, although unlike a strict tit for tat mechanism 
the process leading to high level cooperation is generally robust against following below par 
actions by either player. Rather, high level players would employ coaxing to 
reinforce cooperation in this case. This is true even for lower level players, as after they 
have formed beliefs of the partner, they will not immediately reduce their offers 
upon a few low offers or returns, due to the Bayesian updating mechanism.

The posterior beliefs show both partners ultimately inferring the other's guilt type correctly in both 
pairings, however the $P^I=7$ investors remain aware of the possibility that the partners 
may actually be pragmatic and therefore the high level long horizon investors are prone to 
reduce their offers preemptively towards the end of the game. This data feature was 
noted in particular in the study \cite{Ting2012} and our generative model 
provides a generative explanation for it, based on the posterior beliefs of higher 
level agents explained above.

Figure~\ref{fig:coaxexpl} shows that level $1$ trustees employ coaxing
(returning more than the fair split) to get the investor to give higher
amounts over extended periods of time. In the example settings, the
level $0$ investor completely falls for the trustee's initial coaxing
(left), coming to believe that the trustee is guilty rather than
pragmatic until towards the very end. However, the level $2$ investor
(right) remains cautious and starts reducing offers soon after
the trustee gets greedy, decreasing their offers faster than if playing a truly guilty type.
The level $2$ investor on average remains ambiguous between the partner being guilty or pragmatic. 
Either inference prevents them from being as badly exploited as the level $0$ investor.

In these plots, investor and trustee both have long planning horizons;
we later show what happens when a trustee with a shorter horizon
($P^T=2$) attempts to deceive.
\newpage
\begin{center}
\includegraphics[width=6in, height = 6in]{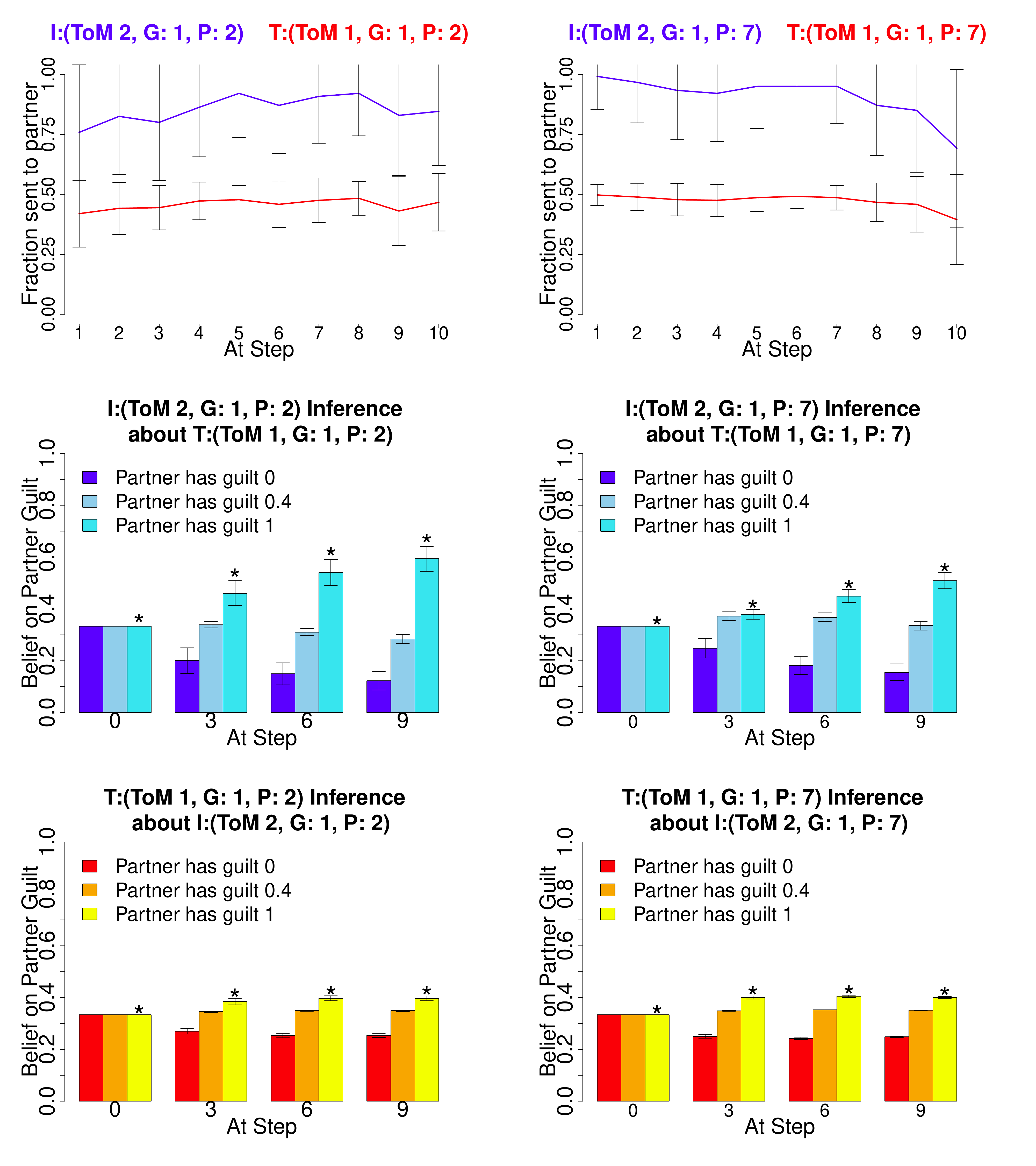}
	\captionof{figure}{\small{Averaged Exchanges (upper) and posteriors
	(mid and lower). Left plots: Investor $(k^I,\alpha^I, P^I)=(2,1,2)$;
	Trustee $(1,1,2)$; right plots: Investor $(2,1,7)$ and Trustee
	$(1,1,7)$. The posterior distributions are shown for
	$\alpha=(0,0.4,1)$ at four stages in the game. Error bars are standard deviations.
	The asterisk denotes the true partner guilt value.}
	\label{fig:cooperation}}
\end{center}

\begin{center}
\includegraphics[width =4 in]{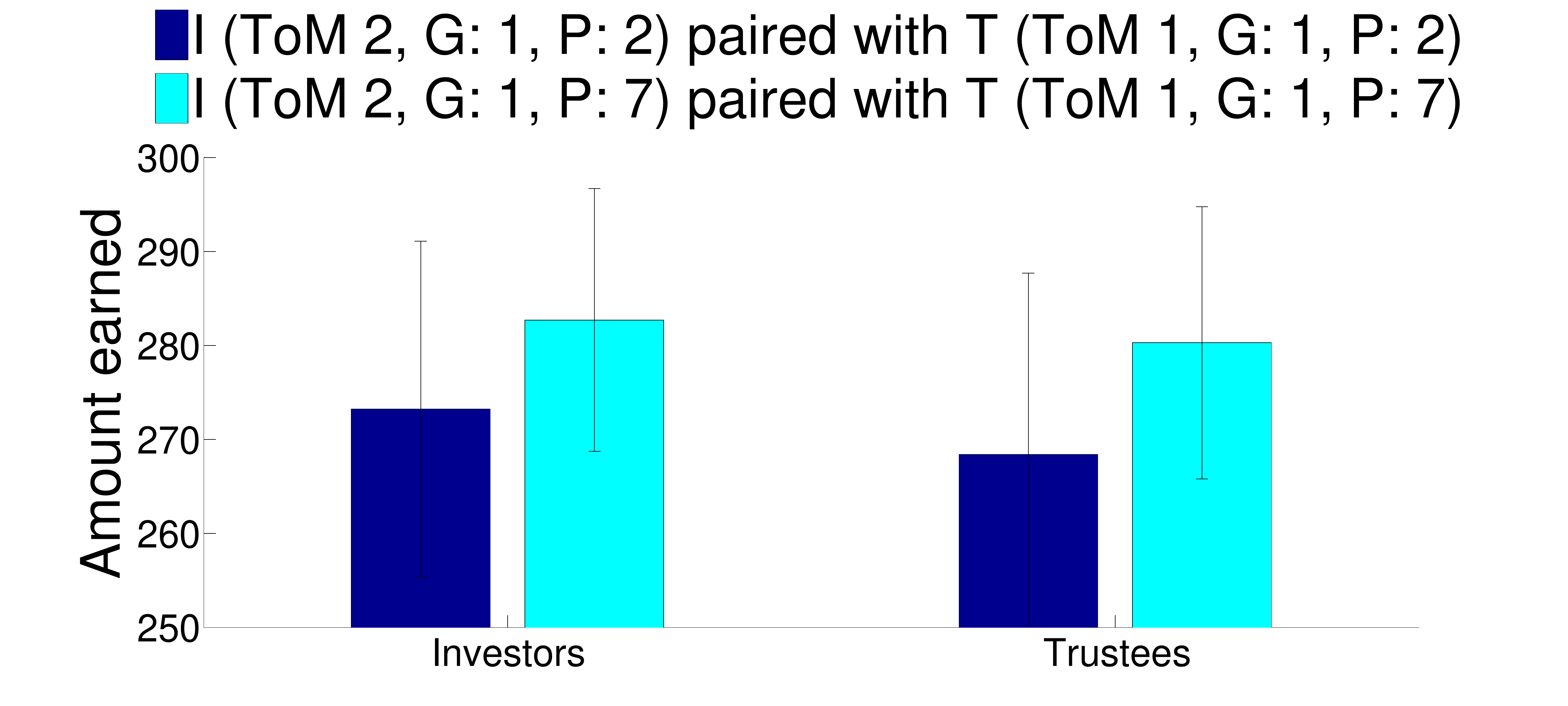}
\captionof{figure}{\small{Average Overall gains for the exchanges in figure
\ref{fig:cooperation}, with planning $2$ (dark blue) and $7$ (light blue). The
difference is highly significant ($p<  0.01$) at a sample size of $60$
for both parameter settings. Error bars are standard deviations.}} 
\label{fig:gains}
\end{center}

A level $1$ trustee can also get pragmatic investors to cooperate through coaxing, as demonstrated in figure \ref{fig:coaxcoop}. The returns are a lot higher than for a level $0$ guilty trustee, who lacks a model of their influence on the investor, and hence does not return enough to drive up cooperation. This initial coaxing is a very common behaviour of high level healthy trustees, trying to get the investor to cooperate more quickly, for both guilty and pragmatic 
high level trustees. 

\begin{center}
	\includegraphics[width=6in, height = 6in]{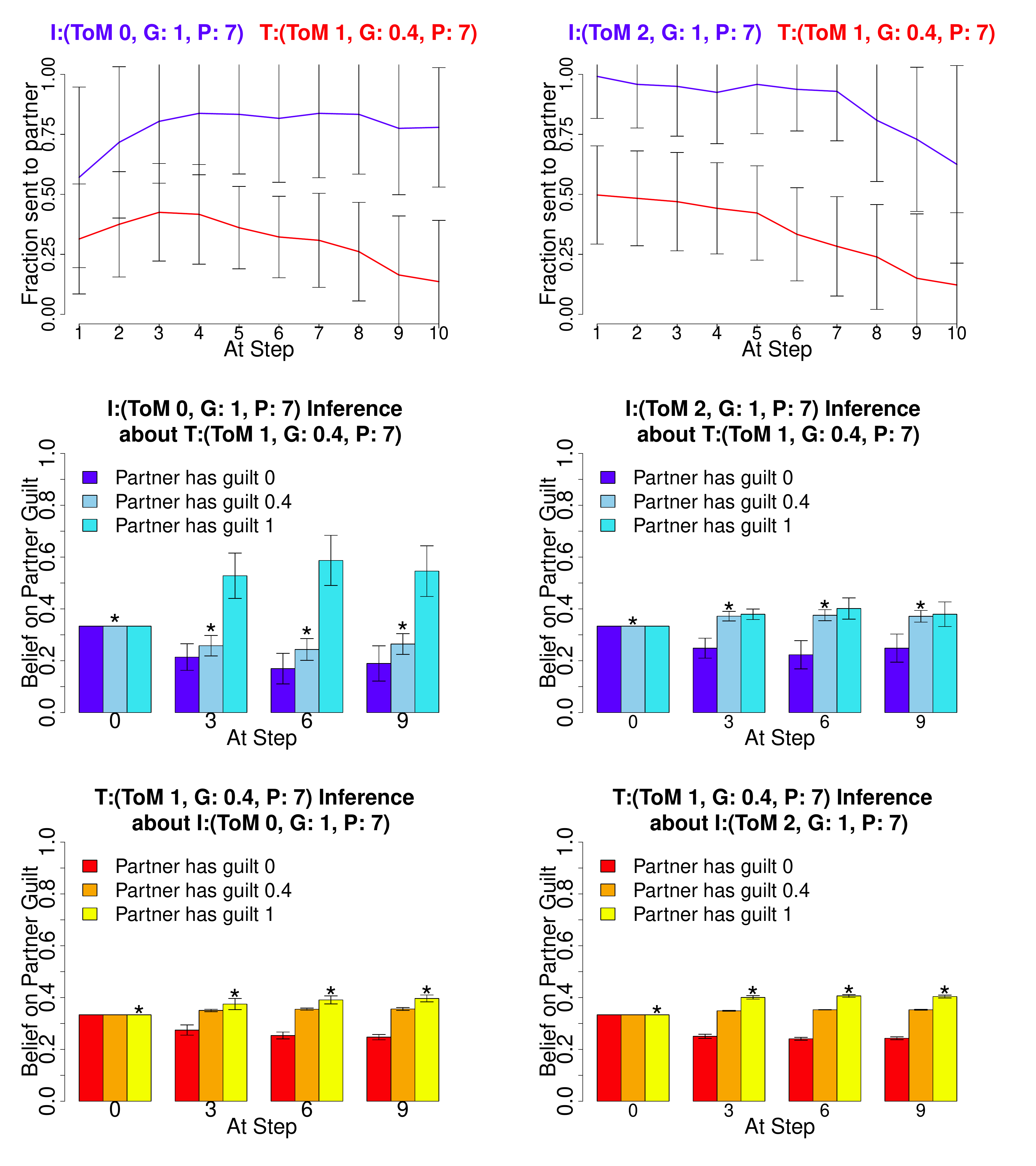}

	\captionof{figure}{\small{Averaged Exchanges (upper) and posteriors
	(mid and lower). Left plots: Investor $(k^I,\alpha^I,P^I)=(0,1,7)$;
	Trustee $(1,0.4, 7)$; right plots: Investor $(2,1, 7)$ and Trustee
	$(1,0.4, 7)$.  The posterior distributions are shown for
	$\alpha=(0,0.4,1)$ at four stages in the game. Error bars are standard deviations.
	The asterisk denotes the true partner guilt value.}
	\label{fig:coaxexpl}}
\end{center}

\vbox{
\begin{center}
\includegraphics[width=3in, height = 6in]{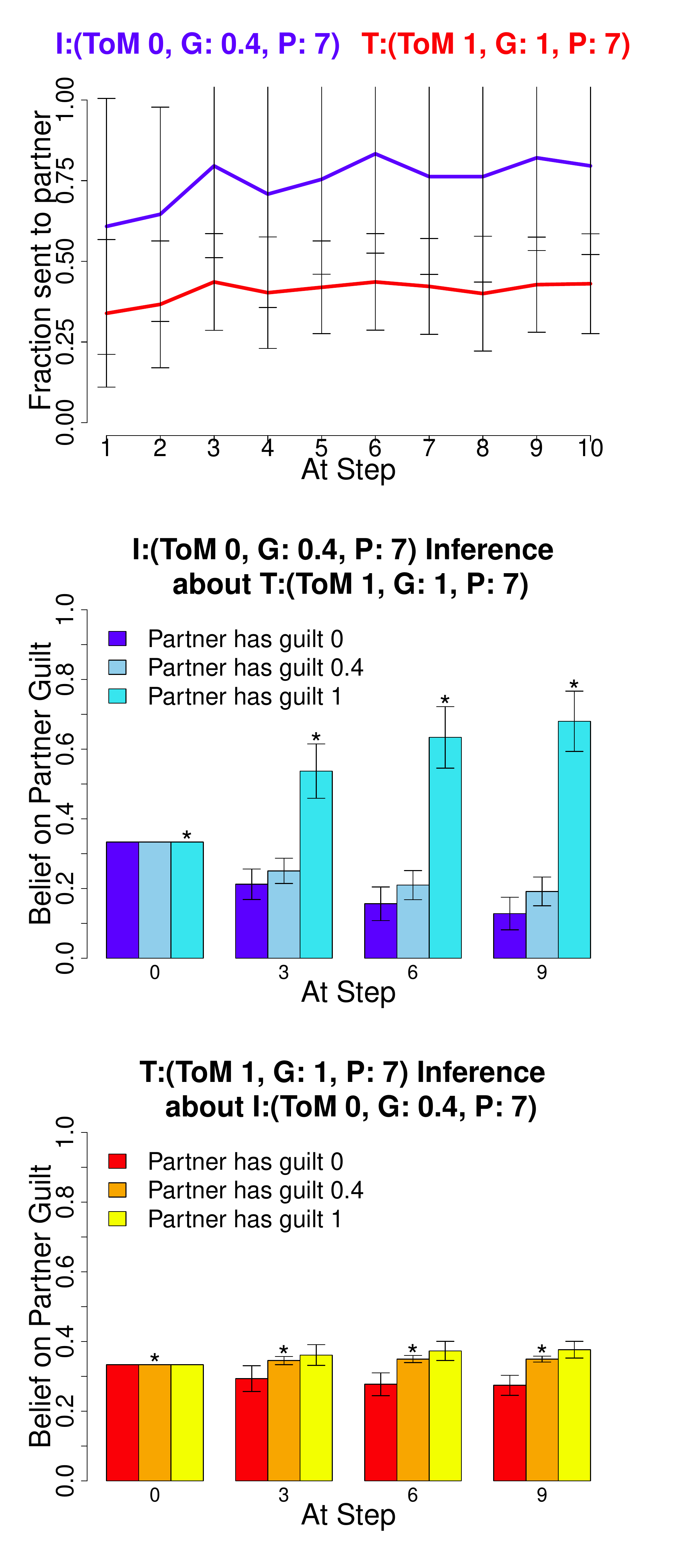}
	
	\captionof{figure}{\small{Average Exchanges (upper) and posteriors (mid and lower), Investor $(0,0.4,7)$ and
	Trustee $(1,1,7)$. The posterior distributions are shown for
	$\alpha=(0,0.4,1)$ at four stages in the game. Error bars are standard deviations.
	The asterisk denotes the true partner guilt value.} }
	\label{fig:coaxcoop}
\end{center}}

\subsection{Inconsistency or Impulsivity}\label{IMPULSIVE}
\vbox{
\begin{center}
\includegraphics[width=3in, height = 6in]{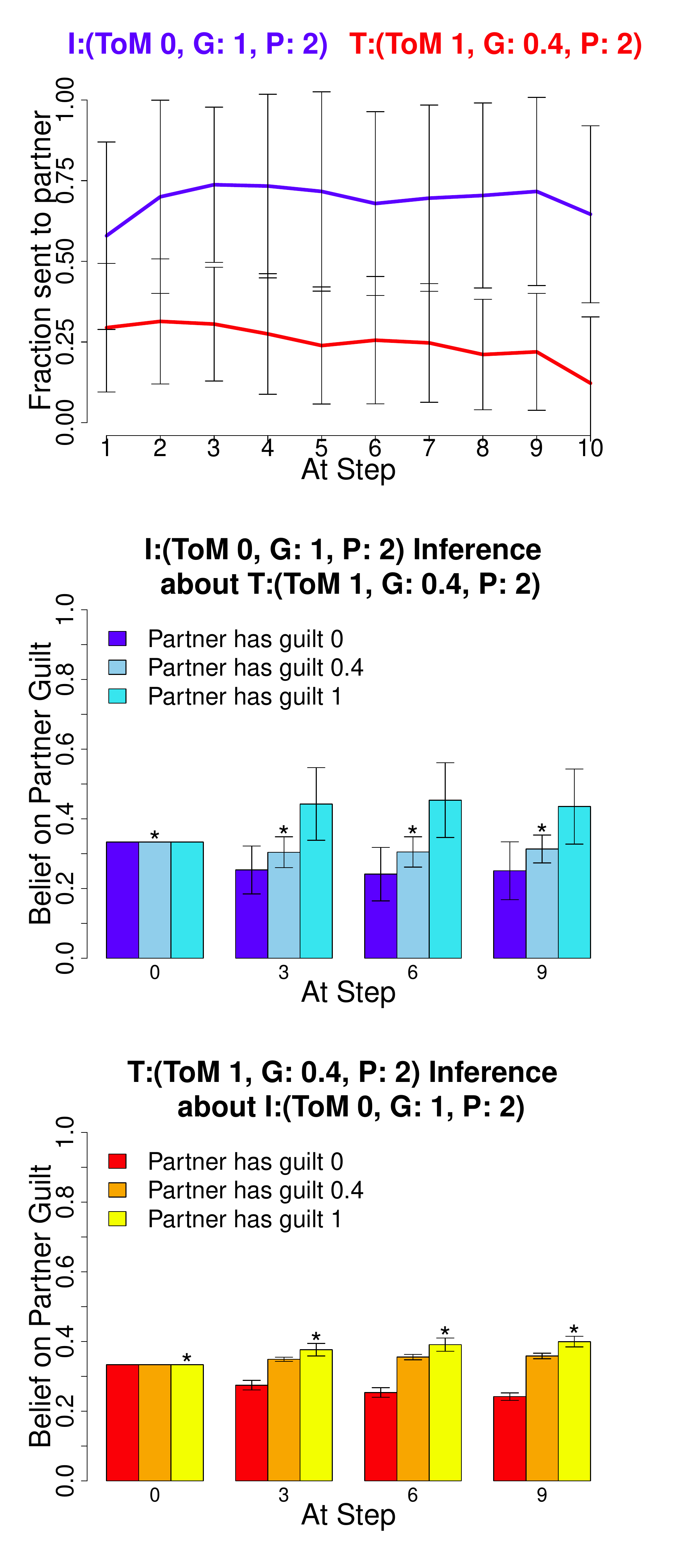}\\
\captionof{figure}{\small{Average Exchanges (upper) and posteriors (mid and lower), Investor $(0,1,2)$ and Trustee $(1,0.4,2)$.
The posterior distributions are shown for
	$\alpha=(0,0.4,1)$ at four stages in the game. Error bars are standard deviations.
	The asterisk denotes the true partner guilt value.}}
\label{fig:deceived}
\end{center}}

Trustees with planning horizon $2$ tend to find it difficult to 
maintain deceptive strategies. As can be seen in
figure \ref{fig:deceived}, even when both agents have a planning horizon
of $2$, a short sighted trustee builds significantly less trust than a long sighted one. This is because it fails to
see sufficiently far in the future, and exploits too early. This
planning horizon thus captures cognitive limitations or impulsive behaviour, while the planning
horizon of $7$ generally describes the consistent execution of a
strategy during play. Such a distinction may be very valuable for the
study of clinical populations suffering from psychiatric disorders such as attention
deficit hyperactivity disorder (ADHD) or borderline personality disorder
(BPD), who might show high level behaviours, but then fail to maintain
them over the course of the entire game. Inferring this requires the
ability to capture long horizons, something that had eluded previous methods.
This type of behaviour shows how important the 
availability of different planning horizons is for modeling, as earlier 
implementations such as \cite{Debbs2008} would treat this impulsive type
 as the default setting.

\subsection{Greedy Behaviour}
\begin{center}

\includegraphics[width=6in, height = 6in]{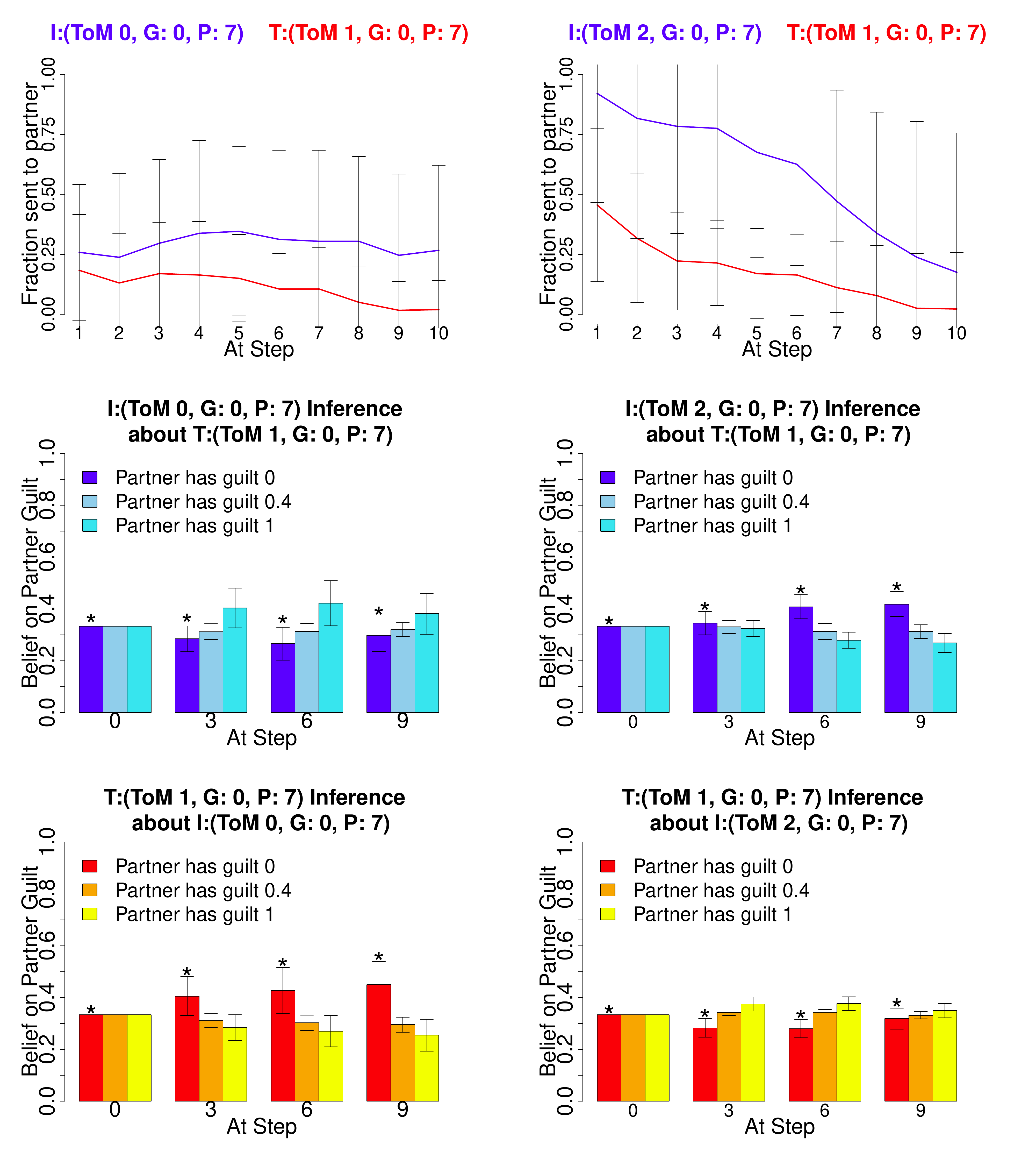}
	\captionof{figure}{\small{Averaged Exchanges (upper) and posteriors
	(mid and lower). Left plots: Investor $(k^I,\alpha^I, P^I)=(0,0,7)$;
	Trustee $(1,0,7)$; right plots: Investor $(2,0,7)$ and Trustee
	$(1,0,7)$.  The posterior distributions are shown for
	$\alpha=(0,0.4,1)$ at four stages in the game. Error bars are standard deviations.
	The asterisk denotes the true partner guilt value.}
	\label{fig:greed}}
\end{center}
Another behavioural phenotype with potential clinical significance arises
with fully greedy partners, see figure \ref{fig:greed}. Greedy low level
investors only invest very little, even if trustees try to convince them of
a high guilt type on their part as described above (coaxing). Cooperation
repeatedly breaks, which is reflected in the high variability of the
investor trajectory.  Two high level greedy types initially cooperate,
but since the greedy trustee egregiously over-exploits, cooperation
usually breaks down quickly over the course of the game, and is not
repaired before the end. In the present context, the greedy type appears
quite pathological in that they seem to hardly care at all about their
partners' type. The main exception to this is the level $2$ greedy
investor (an observation that underscores how theory of mind level and
planning can change behaviour that would seem at first to be hard coded
in the inequality aversion utility function). The level $0$ greedy
investor will cause cooperation to break down, regardless of their
beliefs, as in figure \ref{fig:greed} the posterior beliefs of the level
$0$ show that they believe the trustee to be guilty, but do not alter
their behaviour in the light of this inference.

\subsection{Planning Mismatch - High Level Deceived By Lower Level}\label{MISMATCH}

\begin{center}
\includegraphics[width=3in, height = 6in]{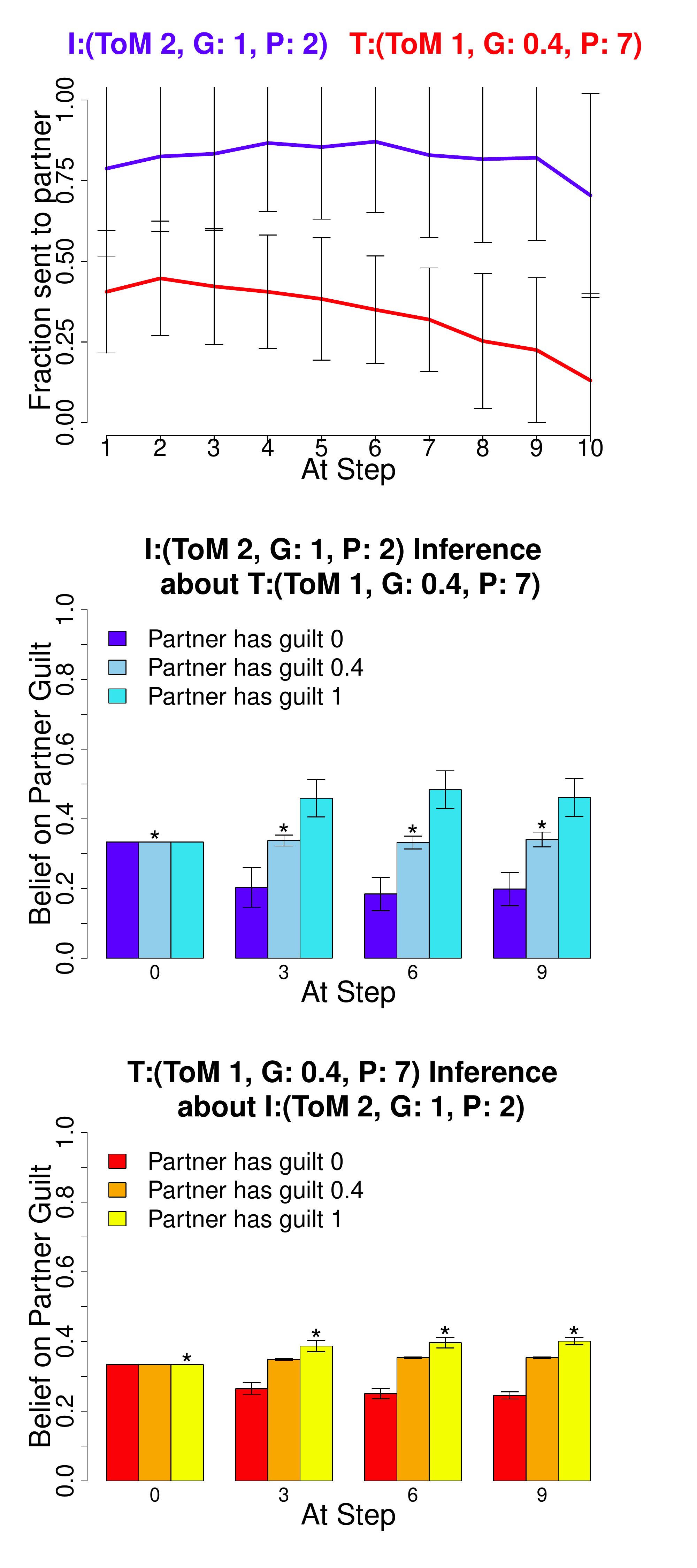}
\captionof{figure}{\small{Average Exchanges, Investor $(2,1,2)$ and Trustee $(1,0.4,7)$.  Error bars are standard deviations.
	The asterisk denotes the true partner guilt value.}}
\label{fig:mismatch}
\end{center}

In figure \ref{fig:mismatch}, the investor is level $2$, and so should have the wherewithal to understand the level $1$ trustee's deception. However, the 
trustee's longer planning horizon permits her to play more consistently, and thus exploit the investor for almost the entire game. This shows that the 
advantage of sophisticated thinking about other agents can be squandered given insufficient planning, and poses an important question about the 
efficient deployment of cognitive resources to the different demands of modeling and planning of social interactions.

\subsection{Confusion}\label{ConfusionDiscussion}

\subsubsection{Model Inversion}

A minimal requirement for using the proposed model to fit experimental data is
self-consistency. That is, it should be possible to recover the
parameters from behaviour that was actually generated from the model
itself. This can alternatively be seen as a test of the statistical
power of the experiment - i.e., whether $10$ rounds suffice in 
order to infer subject parameters.
Figure~\ref{fig:confuse} shows the confusion matrix which indicates the
probabilities of the inferred guilt (top), ToM (middle) and planning
horizon (bottom) for investor (left) and trustee (right), in each case
marginalizing over all the other factors. We discuss a particular special 
case of the obtained confusion in ~\ref{fig:WorstPlan}. Said confusion relates
 to observations made in empirical studies (see \cite{Brooks2005,Ting2012}) and 
suggests the notion of the planning parameter, as measure of consistency of play. 
 Later, we show comparative
data reported in the study \cite{Debbs2008}, which only utilized a fixed
planning horizon of $2$ and $2$ guilt states (and did not exploit the
other simplifications that we introduced above), see figure ~\ref{fig:DebbsConfu} 
for a depiction of the levels of confusion in that study.
These simplifications implied that the earlier study would find recovery of theory of mind 
in particular to be harder. 

\begin{center}

\includegraphics[width=6in, height = 6in]{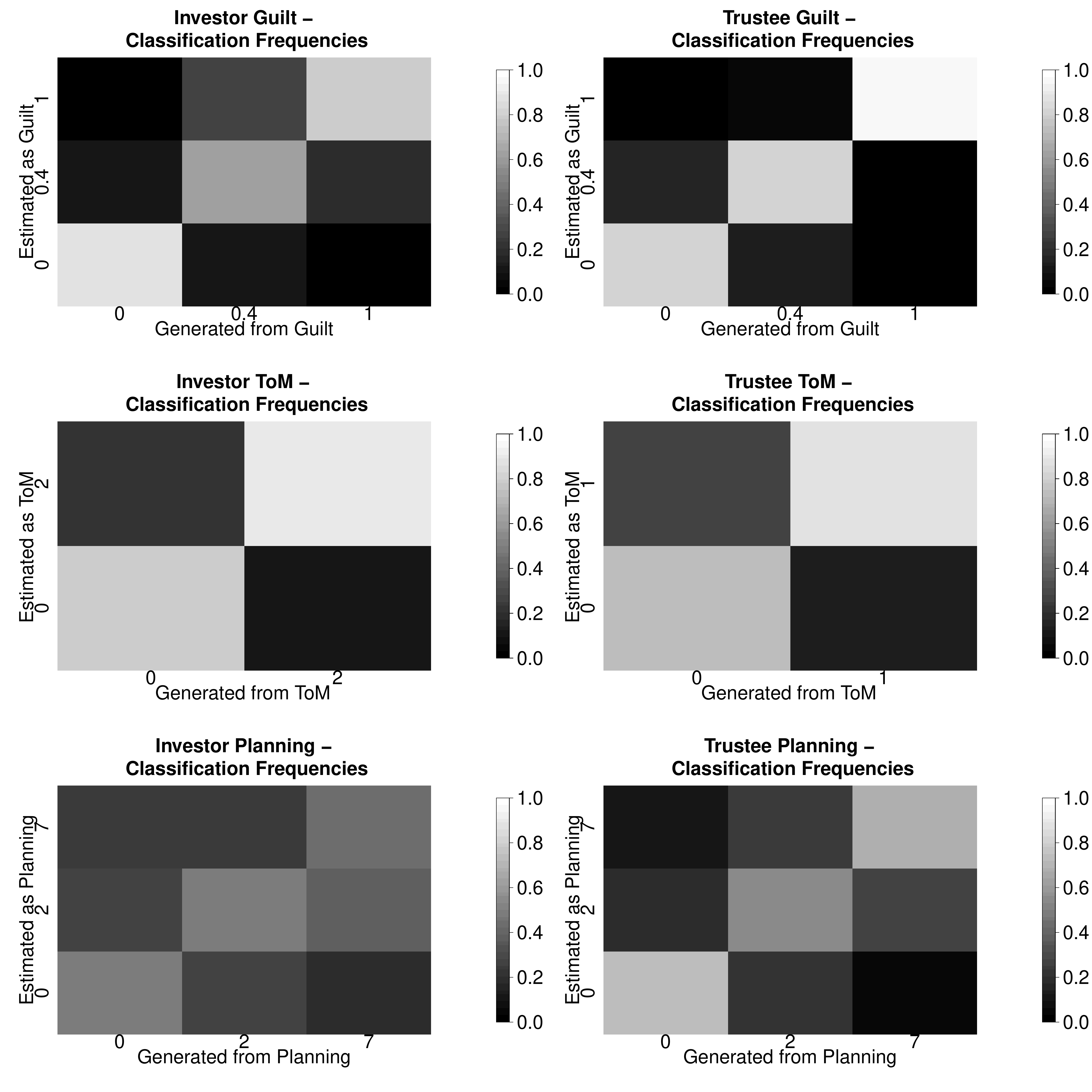}
\captionof{figure}{\small{Percentage of inferred guilt , theory of mind  and planning horizon for investor (left) and trustee (right)
as a function of the true values, marginalizing out all the other
parameters. Each plot corresponds to a uniform mix of $15$ pairs per parameter combination
and partner parameter combination. }
}
\label{fig:confuse}
\end{center}

Guilt is recovered in a highly reliable manner. By contrast, there is a
slight tendency to overestimate ToM in the trustees. The
greatest confusion turns out to be inferring a $P^I=7$ investor as having $P^I=2$
when playing an impulsive trustee ($P^T=2$), a problem shown more
directly in Figure~\ref{fig:WorstPlan}.

 \begin{center}
\vbox{
\includegraphics[width=3in]{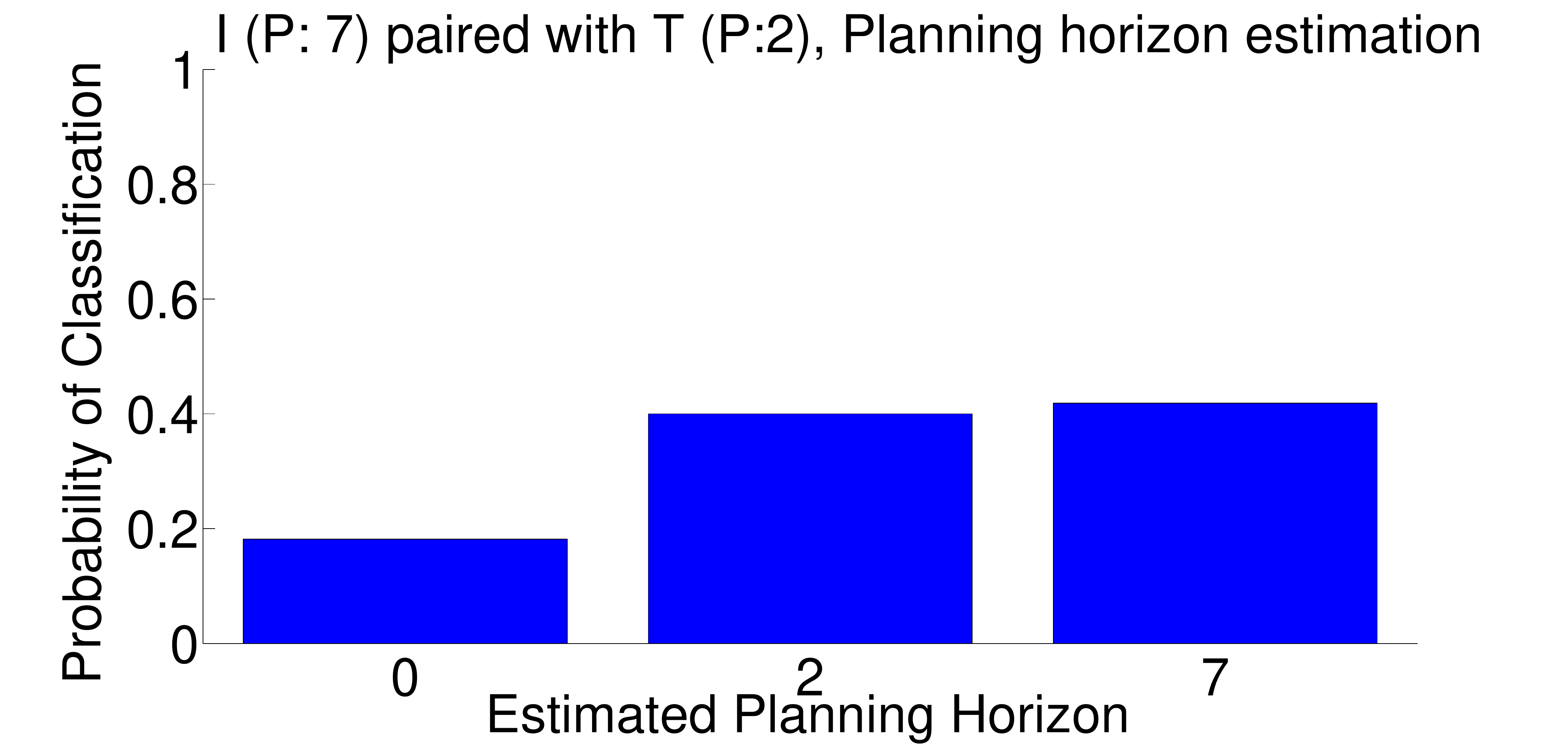}}
\captionof{figure}{\small{Maximum likelihood estimation result, $P^I=7$ and $P^T=2$ agent combinations, marginalized maximum likelihood estimation of 
investor planning horizon over all other parameters.}
}
\label{fig:WorstPlan}
\end{center}
 
The issue is that when the trustee is impulsive, far-sighted investors
($P^I=7$) can gain no advantage over near-sighted ones ($P^I=2$), and so
the choices of this dyad lead to mis-estimation. Alternatively put, an
impulsive trustee brings the investor down to his or her level. This has
been noted in previous empirical studies, notably
\cite{Brooks2005,Ting2012}'s observations of the effect on investors of
playing erratic trustees. The same does not apply on the trustee side,
since the reactive nature of the trustee's tactics makes them far less
sensitive to impulsive investor play.
 
Given the huge computational demands of planning, it seems likely that
investors could react to observing a highly impulsive trustee by
reducing their own actual planning horizons. Thus, the inferential
conclusion shown in figure~\ref{fig:WorstPlan} may in fact not be
erroneous. However, this possibility reminds us of the necessity of
being cautious in making such inferences in a two-player compared to a
one-player setting.

\subsubsection{Confusion Comparison to earlier Work}

We compare our confusion analysis to the one carried out in the grid 
based calculation in 
\cite{Debbs2008}. In \cite{Debbs2008} the authors do not report exact
confusion metrics for the guilt state, only noting that it is possible
to reliably recover whether a subject is characterized by high guilt ($0.7$) or low guilt
($0.3$). We can however compare to the reported ToM level recovery. The
comparison with \cite{Debbs2008} faces an additional difficulty in that
despite using the same formal framework as this present work, the
indistinguishability of the level $1$ and $2$ trustees and the level $0$
and $1$ investors was not identified yet. This explains the somewhat higher
amount of confusion when classifying ToM levels, reported in
$\cite{Debbs2008}$.
Also, since calculation of the Dirichlet-Multinomial
probability was done numerically in this study, some between level
differences will only derive from changes in quadrature points for
higher levels. 
 As can be seen in figure $\ref{fig:DebbsConfu}$ (left), 
almost all of the level $1$ trustees at low guilt are misclassified. This 
 is due to them being classified as level $2$ instead, since both levels 
have the same behavioral features, but apparently the numerical 
calculation of the belief state favored the level $2$ classification over 
the level $1$ classification. The tendency to overestimation is true on the 
investor side as well, with there being a considerable confusion between 
level $0$ and level $1$ investors, who should behaviorally be equivalent. 
In sum, this leads to the reported overestimation of the
theory of mind level. We have depicted the confusion levels reported in
\cite{Debbs2008} in figure $\ref{fig:DebbsConfu}$.

\begin{center}
\includegraphics[width=6in, height = 2in]{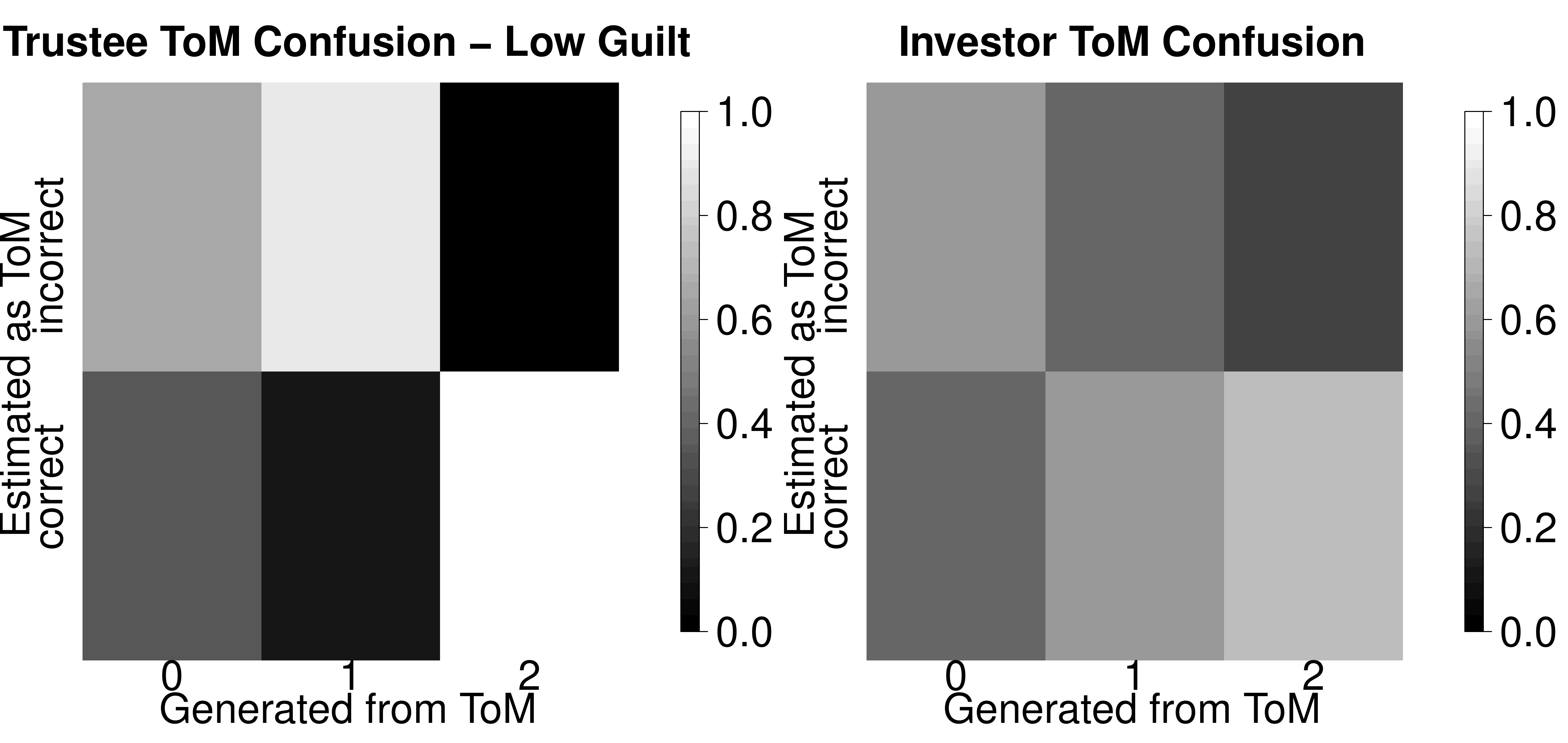} 
	\captionof{figure}{ \small{Classification probability reported in \cite{Debbs2008}. In analogy to figure $\ref{fig:confuse}$ we depict the generated vs estimated values in a matrix scheme.}
	\label{fig:DebbsConfu}}
\end{center}

\end{section}

\subsection{Computational Issues}\label{Computational}

The viability of our method rests on the running time and stability of
the obtained behaviours. In figure \ref{fig:Runtimes}, we show these for
the case of the first action, as a function of the number of simulation
paths used. All these calculations were run at the local Wellcome Trust
Center for Neuroimaging (WTCN) cluster. Local processor cores where of
Intel Xeon E312xx (Sandy Bridge) type clocked at $2.2$ GHz and no
process used more than $4$ GB of RAM. Note that, unless more than
$25k$ paths are used, calculations take less than 
$2$ minutes. 

\begin{center}
\includegraphics[width=6in, height = 2in]{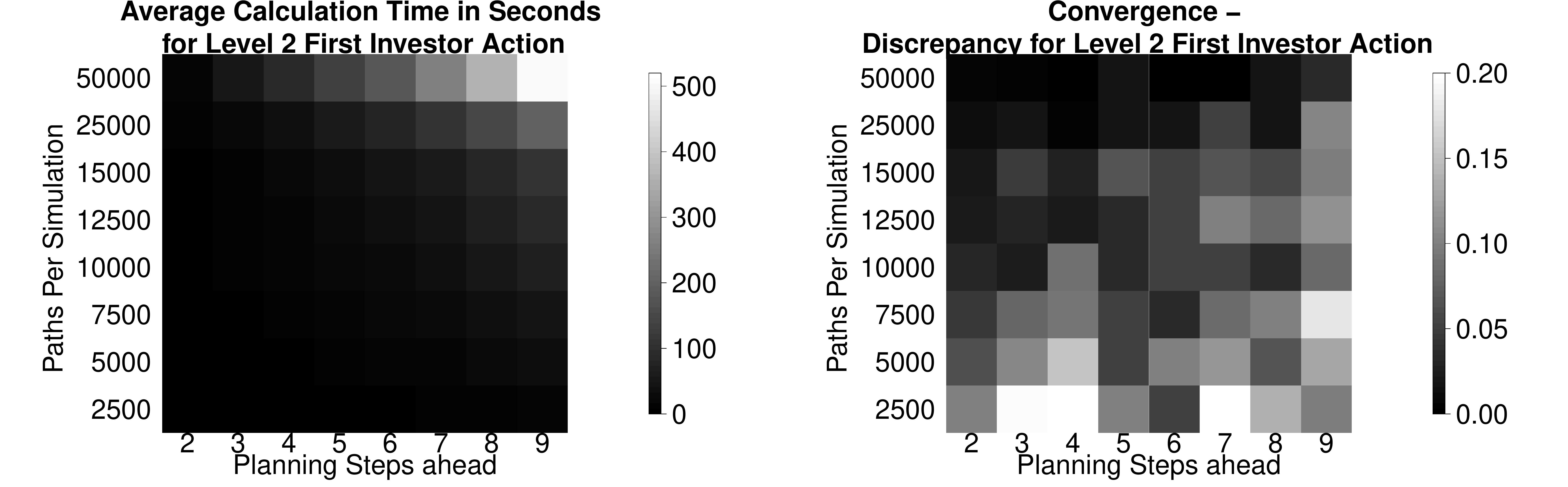}
\captionof{figure}{ \small{(left) Average running times for calculating the first action value of a level $2$, guilt $1$ investor from a given number of simulations, as a function of planning horizon (complexity).
  (right) Discrepancy to the converged
  case of the action probabilities for the first action measured in squared discrepancies.}}
\label{fig:Runtimes}
\end{center}

We quantify simulation stability by comparing simulations for a level
$2$ investor (a reasonable upper bound, because the action value
calculation for this incorporates the level $1$ trustee responses) based
on varying numbers of paths with a simulation involving $10^6$ paths
that has converged. We calculate the between
(simulated) subject discrepancies $C$ of the probabilities for
the first action for $P^I \in \{2 ,3 ,4 ,5 ,6 ,7 ,8 ,9 \}$:
\[
C_{ij} =\frac 1 {119}\sum_{k=1}^{120} (\mathbb P^k[a_0^{I} = \frac i 4 ]
- \hat{\mathbb P} [a^I_0 = \frac i 4] ) (\mathbb P^k[a_0^I = \frac j 4 ]
- \hat{\mathbb P} [a^I_0 = \frac j 4] ) \qquad i,j \in \{ 0, \ldots, 4
\} 
\]
where $\hat{\mathbb P} [a^I_0 = \frac i 4]$ are the converged
probabilities, and $\mathbb P^k[a_0^I]$ is the action likelihood of
simulated subject $k$. If the sum of squares of the entries in the discrepancy matrix is low, then the
probabilities will be close to their converged values.

As can be seen from figure~\ref{fig:Runtimes} (right), for $25$k paths even
planning $9$ steps ahead agents have converged in their initial action
probabilities, such that their action probabilities vary from the
converged value by no more than about $0.1$. However, note that this
convergence is not always monotonic in either the planning horizon or
the number of sample paths. The former is influenced by the differing
complexity of preferences for different horizons -- sometimes, actions
are harder to resolve for short than long horizons. The latter is
influenced by the initial pre-search using constant strategies.

Although $25$k steps suffice for convergence even when planning $9$
steps ahead, this horizon remains computationally challenging. We thus
considered whether it is possible to use a shorter horizon of $7$
steps, without materially changing the preferred
choices. Figure~\ref{fig:SevenNine} illustrates that the difference is
negligible compared with the fluctuations of the Monte Carlo approach,
even for the worst case involving the pairing of $2$ pragmatic types,
with high ToM levels and long planning horizons. At the same time, the calculation for $P=7$ is twice as fast as $P=9$ for the level $2$ investor, which even just for the first action is a difference of $100$ seconds.
\begin{center}
\includegraphics[width=4in, height = 2.in]{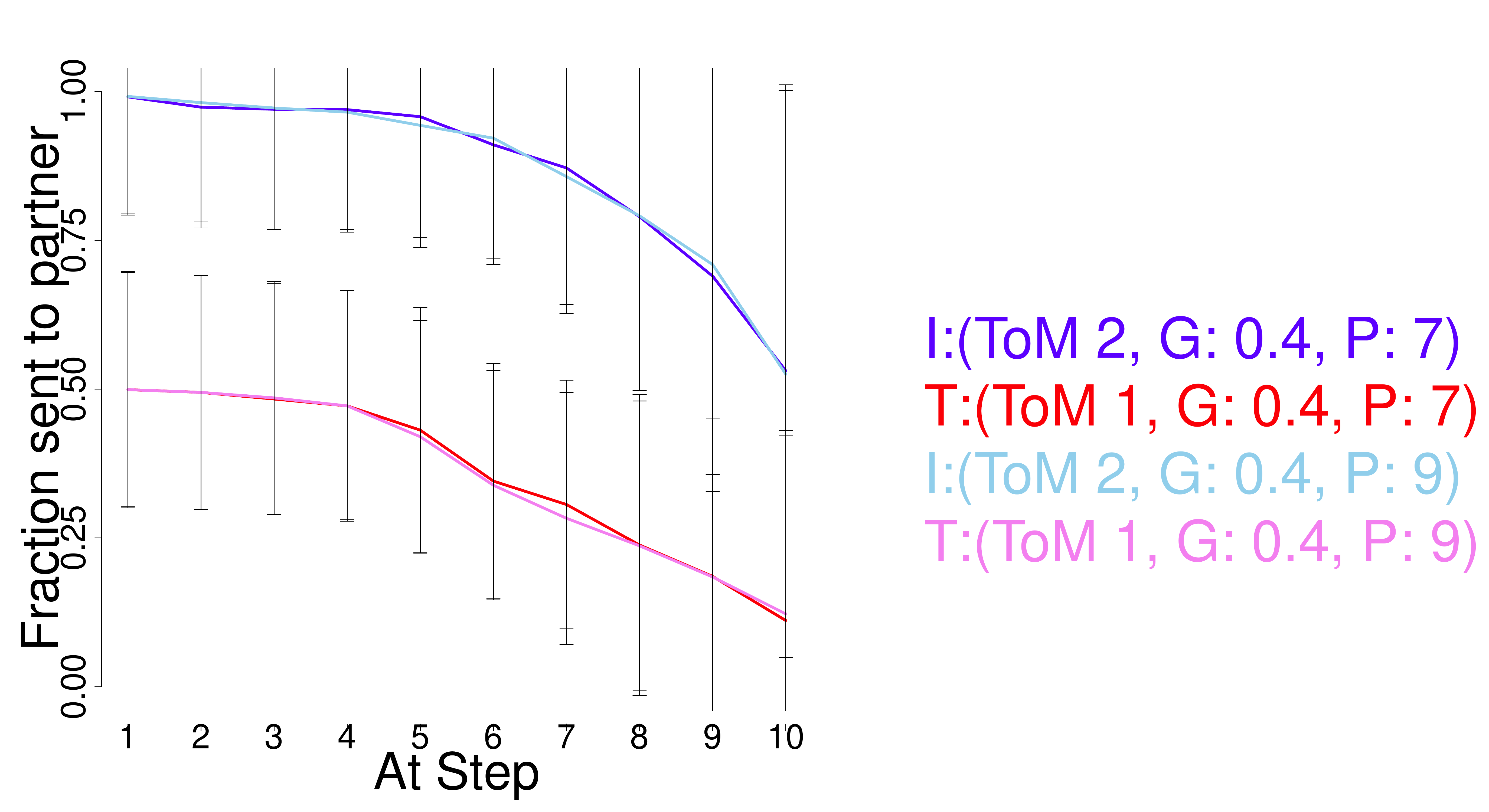}
\captionof{figure}{\small{Average Exchanges, Investor $(2,0.4,7)$ (dark blue) and Trustee $(1,0.4,7)$ (red), as well as Investor $(2, 0.4, 9)$ (light blue) and Trustee $(1, 0.4, 9)$ (rose). The difference between the $2$ planning horizons is not significant at any point.   Error bars are standard deviations.}  }
\label{fig:SevenNine}
\end{center}
Finally we compare our algorithm at planning $2$ steps ahead to the
grid-based calculation used before \cite{Debbs2008,Ting2012}. The
speed advantage is a factor of $200$ for $10^4$ paths in POMCP
demonstrating the considerable improvement that enables us to consider
longer planning horizons.
\subsection{Comparison To Earlier Subject Classifications}\label{EarlierWork}

We will show below, using real subject data taken from \cite{Ting2012}, that our reduction to $3$ guilt states does not render 
likelihoods worse and only serves to improve classification quality. 
We compared the results of our new method with the results obtained in earlier studes (\cite{Debbs2008}, \cite{Ting2012}).
\subsubsection{Dataset}
 We performed inference on the same data sets as in Xiang et
al,\cite{Ting2012} (which were partially analysed in
\cite{Debbs2008},\cite{Brooks2008} and \cite{Misha2010}). This involved
$195$ dyads playing the trust game over $10$ exchanges. The investor
agent was always a healthy subject, the trustees comprised various
clinical groups, including anonymous, healthy trustees (the "impersonal"
group; $48$ subjects), healthy trustees who were briefly encountered
before the experiment (the "personal" group; $52$ subjects), trustees
diagnosed with Borderline Personality Disorder (BPD) (the "BPD" group;
$55$ subjects), and anonymous healthy trustees matched in socio-economic
status (SES) to the (lower than healthy) SES distribution of BPD trustees, (the "low SES'' group; $38$ subjects).

\subsubsection{Models Used}

We compared our models to the results of the model used in \cite{Ting2012} on
the same data set (which incorporates the data set used in \cite{Debbs2008}).
The study \cite{Ting2012} uses $5$ guilt states $\{ 0, 0.25, 0.5, 0.75, 1 \}$ compared to our $3$, 
a planning horizon of $2$ and an inverse temperature of $1$, otherwise the formal framework is exactly the same 
as in section ~\ref{TruPlan}.
Action values in \cite{Ting2012} were calculated by an exact grid search over all 
possible histories and a  
numerical integration for the calculation of the belief state.  
For comparison purposes we built a "clamped" model in which the
planning horizon was fixed at the value $2$, with $3$ guilt states and a inverse temperature set to
$\beta = \frac 1 3$. Additionally, we compared to the outcome for the
full method in this work, including estimation of the planning horizon.  We noted
that in the analysis in \cite{Ting2012}, an additional approximation had been
made at the level $0$ investor level, which set those investors as non
learning.  This kept their beliefs uniform and yielded much
better negative loglikelihoods within said model, than if they were learning.

\subsubsection{Subject Fit}

A minimal requirement to accept subject results as significant is that
the negative log likelihood is significantly better than random on average at $p <0.05$, otherwise we would not trust a
model based analysis over random chance and the estimated parameters
would be unreliable. This criterion is numerically expressed as a negative loglikelihood of $16.1$ for $10$ exchanges, calculated 
from $5$ possible actions at a probability of $0.2$ each, with independent actions each round.

For the analysis in \cite{Ting2012}, we found that the special approximation made in
\cite{Ting2012} allowed for significantly better negative log
likelihoods (mean $11.98$); if this approximation is removed, the
investor data fit at an inverse temperature of $1$ would be worse than
random for this data set. Additionally, the model used in \cite{Ting2012}
did not fit the trustee data significantly better than random at
$p<0.05$ (mean negative loglikelihoods $15.6$ and standard deviation of
$>3$).  

Conversely, for both our clamped and full model analysis at $\beta =
\frac 1 3$, the trustee likelihood is significantly better than random
($11.7$ at the full model) and the investor negative loglikelihood is
slightly better on average (smaller) than found in \cite{Ting2012} with $5$
guilt states ($11.7$ for our method, vs $11.98$).  This confirms that
reducing the number of guilt states to $3$ only reduces confusion and
does not worsen the fit of real subjects data.  Additionally, it becomes
newly possible to perform model-based analyses on the BPD trustee guilt
state distribution, since the old model did not fit trustees
significantly better than random at $p<0.05$.

The seemingly low inverse temperature at $\beta = \frac 1 3$ is a consequence of the 
size of the rewards and the quick accumulation of higher expectation values 
with more planning steps, as the inverse temperature needs to counter balance the 
expectation size to keep choices from becoming nearly deterministic.
Average investor reward expectations (at the first exchange) for planning $0$ steps stand at $18$ with 
an average $18$ being added at each planning step. 

\begin{center}
\includegraphics[width = 6in, height = 6 in]{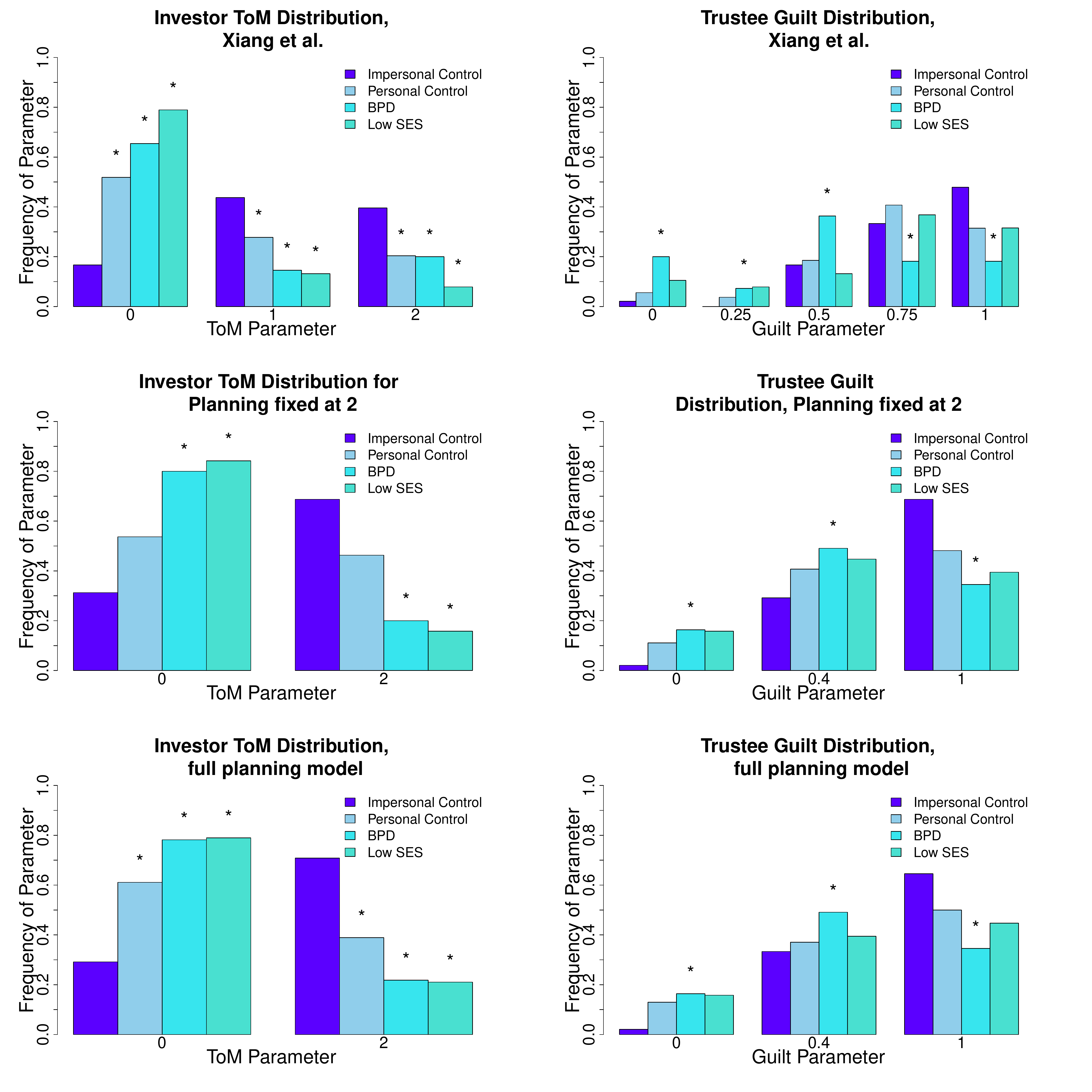}
\captionof{figure}{\small{Parameter Distributions for different models on the data set of ~\cite{Ting2012}. (upper left) Investor ToM distribution is significant ($p<0.05$) between the impersonal control condition 
and all other conditions. (upper right) Trustee Guilt distribution is significant between impersonal controls and the BPD trustees.
(middle left) Planning $2$ investor ToM distribution with $3$ guilt states. BPD and low SES differences to impersonal are significant. (middle right) Planning $2$ trustee guilt, the difference 
between BPD trustees and impersonal controls is significant.
  (bottom left) Full planning model investor ToM, all differences to impersonal are significant. 
(bottom right) Full planning model trustee guilt. BPD trustees are significantly different from controls.
The asterisk denotes a significant ($p<0.05$) difference in the Kolmogorov-Smirnov two sample test, to the impersonal control group.}
\label{fig:Xiang}}
\end{center}

\subsubsection{Marginal Parameter Distributions -Significant Features}

Figure $\ref{fig:Xiang}$ shows the significant parameter distribution differences (Kolmogorov-Smirnov 
two sample test, $p <0.05$). For investor theory of mind and trustee guilt distribution, many of the same differences are significant for the analysis reported in ~\cite{Ting2012} (see Fig. ~\ref{fig:Xiang}, upper panels), 
for an analysis using our model with a "clamped" planning horizon of $2$ steps 
ahead (see Fig. ~\ref{fig:Xiang}, middle panels, to match with the approach of ~\cite{Debbs2008}) and for our full model, using $3$ guilt states, ToM level up to $2$ and $3$ planning horizons (see Fig. ~\ref{fig:Xiang},
 bottom panels and Fig. ~\ref{fig:plan}).  We find significantly lowered ToM in most other groups, 
compared to the impersonal control group. We find a significantly lowered guilt distribution in BPD trustees, however the guilt difference was not used for fMRI analysis in 
~\cite{Ting2012}, because, as noted above, the trustee was not fit significantly better than random at $p <0.05$ in the earlier model. For our full model with $3$ planning 
values, we find additional significant differences on the investor side: While all ToM distributions are 
significantly different from the impersonal condition, the planning difference between the personal and impersonal conditions is not significant at $p<0.05$, while it is significant for the other groups 
(see Fig. ~\ref{fig:plan}). Thus, this is the only model keeping the parameter distribution of the personal group distinct from both the impersonal group 
(from which it is not significantly different in the clamped model) 
and the low SES playing controls and BPD playing controls (from which it is not significantly different based on the parameters in ~\cite{Ting2012}) at the same time.

This supports the planning horizon as a "consistency of play" and additional rationality measure, as the subjects do not think about possible partner deceptions as much in the personal condition, 
having just met the person they will be playing (resulting in lowered ToM). However, their play is non disruptive, if low level, and consistent exchanges result. BPD and low SES trustees however 
disrupt the partners' play, lowering their planning horizon.

\begin{center}
\includegraphics[width=3in, height = 2in]{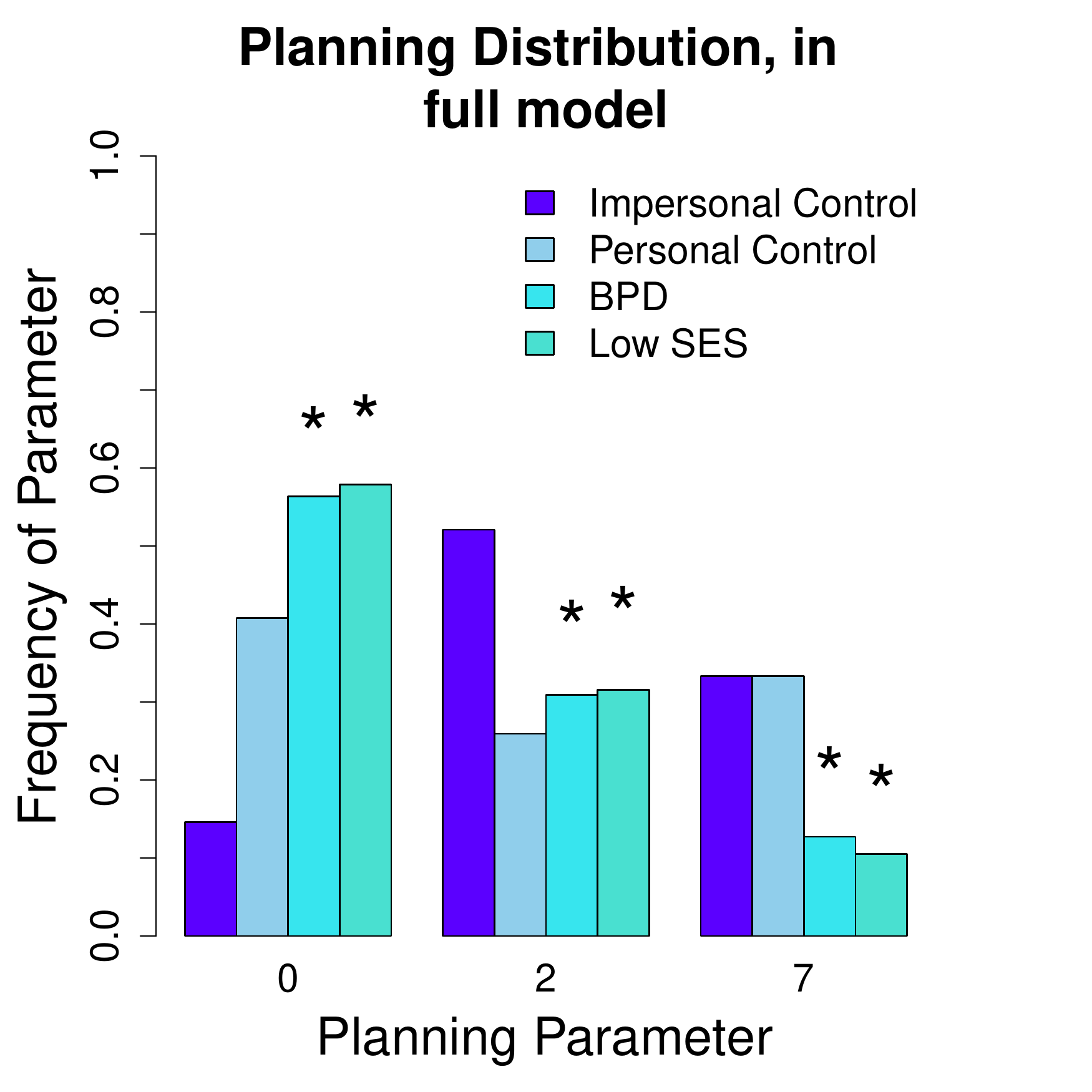}
	\captionof{figure}{\small{Planning distribution for Investors, distinguished between personal 
condition controls (non significant) and BPD and low SES trustees (significantly lower than impersonal).
The asterisk denotes a significant ($p<0.05$) difference in the Kolmogorov-Smirnov two sample test,
 to the impersonal control group.}
\label{fig:plan}}
\end{center}

\begin{section}{Discussion}

We adapted the Monte-Carlo tree search algorithm designed for partially
observable Markov decision processes  \cite{POMCP2010} to the interactive,
game-theoretic, case \cite{IPOMDP}. We provide significant simplifications to the case of dyadic 
social exchange, which benefit any IPOMDP based method.
We illustrated the power of 
this method by extending the
computationally viable planning horizon in a complex, multi-round,
social exchange game to be able to encompass characteristic behaviours that have been seen in
human play \cite{Brooks2008}. 

We also showed that the 10 rounds that had
been used empirically suffice to license high quality inference about
parameter values, at least in the case that the behaviour was generated
from the model itself. We exhibited three fundamental forms of dynamical
behaviour in the task: cooperation, and two different varieties of
coaxing. The algorithm generates values, state-action values and
posterior beliefs, all of which can be used for such methods as
model-based fMRI. 

We find that the results in  \ref{IMPULSIVE}, \ref{MISMATCH} and figures \ref{fig:WorstPlan} 
and \ref{fig:plan} confirm the planning horizon as a consistency of 
play parameter, that encodes the capability of a subject to execute 
a consistent strategy throughout play. As such it may be disrupted 
by the behavior of shorter planning 
partners, as can be seen in ~\ref{fig:WorstPlan} and \ref{fig:plan}.

Furthermore, comparing to earlier data used in the work \cite{Ting2012} 
we can confirm the relevance of the planning parameter in the treatment 
of real subject data, classifying subject groups along the new axis 
of consistency of play.

The newly finer classification of subjects along the three axes of
theory of mind, planning horizon and guilt $(k, P, \alpha)$ should
provide a rich framework to classify deficits in clinical populations
such as an inability to model other people's beliefs or intentions,
ineffective model-based reasoning, and a lack of empathy. Such analyses
can be done at speed, of the order of 10s of subjects per hour. 

One might ask whether the behavioural patterns derived in this work 
might be obtained without invoking the cognitive hierarchy and 
instead using a large enough state space, which encodes the
 preferences and sophistication of the other agent as many separate states, rather than 
a few type parameters plus the cognitive hierarchy. This is in 
principle possible, however we prefer ToM for $2$ reasons: Firstly, 
the previous study \cite{Ting2012} and others have found neural 
support for the distinction between high ToM and low ToM subjects in real play, 
suggesting that this distinction is not but a mathematical convenience (cf. \cite{Ting2012},  
p.$4$ and $5$ for a neural representation 
of prediction errors associated to level $0$ and level $2$ thinking).
Secondly, we can specify features of interest, such as inequality aversion 
and planning at the lowest level, then generate high level behaviours in 
a way that yields an immediate psychological interpretation in terms of the 
mentalization steps encoded in the ToM level. 

The algorithm opens the door to finer analysis of complicated social
exchanges, possibly allowing optimization over initial prior values in
the estimation or the analysis of higher levels of theory of mind, at
least on tasks with lower fan-out in the search tree. It would also be
possible to search over the inverse temperature $\beta$.

One important lacuna is that although it is straightforward to use
maximum likelihood to search over fixed parameters (such as ToM level,
planning horizon or indeed temperature), it is radically harder to
perform the computations that become necessary when these factors are
incorporated into the structure of the intentional models. That is, our
subjects were assumed to make inferences about their opponent's guilt,
but not about their theory of mind level or planning horizon. 

 It is possible that additional tricks would make this viable for the trust
task, but it seems more promising to devise or exploit a simpler game in
which this would be more straightforward.

\end{section}

\begin{section}{Acknowledgements}
  The authors would like to thank James Lu, Johannes Heinrich, Terry Lohrenz and Arthur
  Guez for helpful discussions and Xiaosi Gu, Michael Moutoussis, Tobias Nolte and Iris Vilares for comments on the manuscript. Special thanks go to Andreas Morhammer,
  who brilliantly advised on several issues with C++, as well as the IT
  support staff at the Wellcome Trust Center for Neuroimaging and
  Virgina Tech Carilion Research Institute. The authors gratefully acknowledge funding by 
the Wellcome Trust (Read Montague) under a Principal Research Fellowship, the Kane Foundation (Read Montague) and 
the Gatsby Charitable Foundation (Peter Dayan). Andreas Hula is supported by the Principal Research Fellowship of Professor 
Read Montague.
\end{section}


\bibliography{NeurBibAlt}
\bibliographystyle{unsrt}{}

\newpage
\appendix

\end{document}